\documentclass[11pt, a4paper]{article}

\usepackage[T1]{fontenc}
\usepackage[scaled=1.05,helvratio=0.95]{newtxtext}
\usepackage{color}
\usepackage{graphicx}
\usepackage{amsmath,amsthm,amssymb}
\usepackage{enumerate}
\usepackage{multirow}
\usepackage{subfigure} 
\usepackage{float} 
\usepackage{algorithm,algorithmic}

\newtheorem{theorem}{Theorem}[section]
\newtheorem*{theorem*}{Theorem}
\newtheorem{lemma}[theorem]{Lemma}
\newtheorem*{lemma*}{Lemma}
\newtheorem{corollary}[theorem]{Corollary}
\newtheorem*{corollary*}{Corollary}
\newtheorem{definition}[theorem]{Definition}
\newtheorem*{definition*}{Definition}
\newtheorem{remark}[theorem]{Remark}
\newtheorem*{remark*}{Remark}

\renewcommand{\thesection}{{{\arabic{section}}}}

\renewcommand{\theequation}{{{\thesection.\arabic{equation}}}}

\providecommand{\keywords}[1]
{
  \small	
  \textbf{Keywords: } #1
}

\title{
    Universal Approximation Properties for an ODENet and a ResNet:
    Mathematical Analysis and Numerical Experiments
}
\author{Yuto Aizawa\thanks{Division of Mathematical and Physical Sciences, Kanazawa University (ay.futsal.univ@gmail.com)} \and 
Masato Kimura\thanks{Faculty of Mathematics and Physics, Kanazawa University (mkimura@se.kanazawa-u.ac.jp)} \and
Kazunori Matsui\thanks{Faculty of Science and Technology, Seikei University (kazunori-matsui@st.seikei.ac.jp)}}
\date{}

\begin{document}
    
\maketitle

\begin{abstract}
    We prove a universal approximation property (UAP) for a class of ODENet and a class of ResNet, which are simplified mathematical models for deep learning systems with skip connections. The UAP can be stated as follows. Let $n$ and $m$ be the dimension of input and output data, and assume $m\leq n$. Then we show that ODENet of width $n+m$ with any non-polynomial continuous activation function can approximate any continuous function on a compact subset on $\mathbb{R}^n$. We also show that ResNet has the same property as the depth tends to infinity. Furthermore, we derive the gradient of a loss function explicitly with respect to a certain tuning variable. We use this to construct a learning algorithm for ODENet. To demonstrate the usefulness of this algorithm, we apply it to a regression problem, a binary classification, and a multinomial classification in MNIST.
\end{abstract}


\section{Introduction}\label{sec1}
\setcounter{equation}{0}
Recent advances in neural networks have proven immensely successful for regression analysis, image classification, time series modeling, and so on \cite{schmidhuber2015deep}. Neural Networks are models of the human brain and vision \cite{mcculloch1943logical,fukushima1982neocognitron}. A neural network performs regression analysis, image classification, and time series modeling by performing a series of sequential operations, known as layers. Each of these layers is composed of \textit{neurons} that are connected to \textit{neurons} of other (typically, adjacent) layers. We consider a neural network with $L+1$ layers, where the input layer is layer $0$, the output layer is layer $L$, and the number of nodes in layer $l~(l=0,1,\ldots,L)$ is $n_l\in\mathbb{N}$. Let $f^{(l)}:\mathbb{R}^{n_l}\to\mathbb{R}^{n_{l+1}}$ be the function of each layer. The output of each layer is, therefore, a vector in $\mathbb{R}^{n_{l+1}}$. If the input data is $\xi\in\mathbb{R}^{n_0}$, then, at each layer, we have
\begin{equation*}
    \left\{\begin{aligned}
        x^{(l+1)} &= f^{(l)}(x^{(l)}), & l=0,1,\ldots,L-1, \\
        x^{(0)} &= \xi. &
    \end{aligned}\right.
\end{equation*}
The final output of the network then becomes $x^{(L)}$, and the network is represented by $H=[\xi\mapsto x^{(L)}]$.

A neural network approaches the regression and classification problem in two steps. Firstly, a \textit{priori} observed and classified data is used to train the network. Then, the trained network is used to predict the rest of the data. Let $D\subset\mathbb{R}^{n_0}$ be the set of input data, and $F:D\to\mathbb{R}^{n_L}$ be the target function. In the training step, the training data $\{(\xi^{(k)},F(\xi^{(k)}))\}_{k=1}^K$ are available, where $\{\xi^{(k)}\}_{k=1}^K\subset D$ are the inputs, and $\{F(\xi^{(k)})\}_{k=1}^K\subset\mathbb{R}^{n_L}$ are the outputs. The goal is to learn the neural network so that $H(\xi)$ approximates $F(\xi)$. This is achieved by minimizing a loss function that represents a similarity distance measure between the two quantities. In this paper, we consider the loss function with the mean square error
\begin{equation*}
    \frac{1}{K}\sum_{k=1}^K\left|H(\xi^{(k)})-F(\xi^{(k)})\right|^2.
\end{equation*}

Finding the optimal functions $f^{(l)}:\mathbb{R}^{n_l}\to\mathbb{R}^{n_{l+1}}$ out of all possible such functions is challenging. In addition, this includes a risk of overfitting because of the high number of available degrees of freedom. We restrict the functions to the following form:
\begin{equation}\label{eq:intro-function-form}
    f^{(l)}(x)=a^{(l)}\odot\mbox{\boldmath $\sigma$}(W^{(l)}x+b^{(l)}),
\end{equation}
where $W^{(l)}\in\mathbb{R}^{n_{l+1}\times n_l}$ is a weight matrix, $b\in\mathbb{R}^{n_{l+1}}$ is a bias vector, and $a^{(l)}\in\mathbb{R}^{n_{l+1}}$ is weight vector of the output of each layer. The operator $\odot$ denotes the Hadamard product (element-wise product) of two vectors defined by \eqref{eq:Hadamard-product}. The function $\mbox{\boldmath $\sigma$}:\mathbb{R}^{n_{l+1}}\to\mathbb{R}^{n_{l+1}}$ is defined by $\mbox{\boldmath $\sigma$}(x)=(\sigma(x_1),\sigma(x_2),\ldots,\sigma(x_{n_{l+1}}))^{\top}$, where $\sigma:\mathbb{R}\to\mathbb{R}$ is called an activation function. For a scalar $x\in\mathbb{R}$, the sigmoid function $\sigma(x)=(1+e^{-x})^{-1}$, the hyperbolic tangent function $\sigma(x)=\tanh(x)$, the rectified linear unit (ReLU) function $\sigma(x)=\max(0,x)$, and the linear function $\sigma(x)=x$, and so on, are used as activation functions.

If we restrict the functions of the form \eqref{eq:intro-function-form}, the goal is to learn $W^{(l)},b^{(l)},a^{(l)}$ that approximates $F(\xi)$ in the training step. The gradient method is used for training. Let $G_{W^{(l)}},G_{b^{(l)}}$ and $G_{a^{(l)}}$ be the gradient of the loss function with respect to $W^{(l)},b^{(l)}$ and $a^{(l)}$, respectively, and let $\tau>0$ be the learning rate. Using the gradient method, the weights and biases are updated as follows:
\begin{equation*}
    W^{(l)}\leftarrow W^{(l)}-\tau G_{W^{(l)}},\quad b^{(l)}\leftarrow b^{(l)}-\tau G_{b^{(l)}},\quad a^{(l)}\leftarrow a^{(l)}-\tau G_{a^{(l)}}.
\end{equation*}
Note that the stochastic gradient method \cite{bottou1998online} has been widely used recently. Then, error backpropagation \cite{rumelhart1986learning} was used to find the gradient.

It is known that deep (convolutional) neural networks are of great importance in image recognition \cite{simonyan2014very, szegedy2015going}. In \cite{he2015convolutional}, it was found through controlled experiments that the increase of depth in networks actually improves its performance and accuracy, in exchange, of course, for additional time complexity. However, in the case that the depth is overly increased, the accuracy might get stagnant or even degraded \cite{he2015convolutional}. In addition, considering deeper networks may impede the learning process, which is due to the vanishing or exploding of the gradient \cite{bengio1994learning,glorot2010understanding}. Apparently, deeper neural networks are more difficult to train. To address such an issue, the authors in \cite{he2016deep} recommended the use of residual learning to facilitate the training of networks that are considerably deeper than those used previously. Such a network is referred to as residual network or ResNet. Let $n$ and $m$ be the dimensions of the input and output data. Let $N$ be the number of nodes in each layer. A ResNet can be represented as
\begin{equation}\label{eq:intro-resnet}
    \left\{\begin{aligned}
        x^{(l+1)} &= x^{(l)}+f^{(l)}(x^{(l)}), & l=0,1,\ldots,L-1, \\
        x^{(0)} &= Q\xi. &
    \end{aligned}\right.
\end{equation}
The final output of the network then becomes $H(\xi):=Px^{(L)}$, where $P\in\mathbb{R}^{m\times N}$ and $Q\in\mathbb{R}^{N\times n}$. Moreover, the function $f^{(l)}$ is learned from training data.

Transforming \eqref{eq:intro-resnet} into
\begin{equation}\label{eq:intro-euler-method}
    \left\{\begin{aligned}
        x^{(l+1)} &= x^{(l)}+hf^{(l)}(x^{(l)}), & l=0,1,\ldots,L-1, \\
        x^{(0)} &= Q\xi, &
    \end{aligned}\right.
\end{equation}
where $h$ is the step size of the layer, leads to the same equation for the Euler method, which is a method for finding numerical solution to initial value problem for ordinary differential equation. Indeed, putting $x(t):=x^{(l)}$ and $f(t,x):=f^{(l)}(x)$, where $t=hl$, $T=hL$ and $f:[0,T]\times\mathbb{R}^N\rightarrow\mathbb{R}^N$, then the limit of \eqref{eq:intro-euler-method} as $h$ approaches zero yields the following initial value problem of ordinary differential equation
\begin{equation}\label{eq:intro-odenet}
    \left\{\begin{aligned}
        x'(t) &= f(t,x(t)), & t\in(0,T], \\
        x(0) &= Q\xi. &
    \end{aligned}\right.
\end{equation}
We call the function $H=[D\ni\xi\mapsto Px(T)]$ an ODENet \cite{chen2018neural} associated with the system of ordinary differential equations \eqref{eq:intro-odenet}. 

\begin{remark}
    In the real deep learning system, the vector field $f(t,x)$ should be chosen from a family of the vector fields $f\in \{f_\omega\}_\omega$, where $\omega$ is a parameter that is optimized. In this paper, we consider an ODENet associated with \eqref{eq:odenet-main}, which we call an $(\alpha,\beta,\gamma)$-type ODENet. Instead of the variable $x(t)\in \mathbb{R}^N$ in \eqref{eq:intro-odenet}, we consider $(x(t), y(t))\in \mathbb{R}^n\times \mathbb{R}^m$ in \eqref{eq:odenet-main}, where $N=m+n$ (see Appendix \ref{appendix4} for the detail). The implementation of ODENet requires (forward) Euler discretization to ResNet. A ResNet with \eqref{eq:resnet-main} corresponds to the discretized version of the $(\alpha,\beta,\gamma)$-type ODENet, and we call it an $(\alpha,\beta,\gamma)$-type ResNet.
\end{remark}

A neural network of arbitrary width and bounded depth has universal approximation property (UAP). The classical UAP is that continuous functions on a compact subset on $\mathbb{R}^n$ can be approximated by a linear combination of activation functions. It has been shown that the UAP for the neural networks holds by choosing a sigmoidal function \cite{cybenko1989approximation,hornik1989multilayer,carroll1989construction,funahashi1989approximate}, any bounded function that is not a polynomial function \cite{leshno1993multilayer,attali1997approximations}, and any function in Lizorkin space including a ReLU function \cite{sonoda2017neural} as an activation function. The UAP for neural network and its proof for each activation function is presented in Table \ref{table:classical-UAP}.

\begin{table}[ht]
    \begin{center}
        \caption{Activation function and classical universal approximation property of neural network}
        \label{table:classical-UAP}
        \begin{tabular}{lll}
            References & Activation function & How to prove \\ \hline
            Cybenko \cite{cybenko1989approximation} & Continuous sigmoidal & Hahn-Banach theorem \\
            Hornik et al. \cite{hornik1989multilayer} & Monotonic sigmoidal & Stone-Weierstrass theorem \\
            Carroll, Dickinson \cite{carroll1989construction} & Continuous sigmoidal & Radon transform \\
            Funahashi \cite{funahashi1989approximate} & Monotonic sigmoidal & Fourier transform \\
            Leshno et al. \cite{leshno1993multilayer} & Non-polynomial & Weierstrass theorem \\
            Attali, Pag\`es \cite{attali1997approximations} & Non-polynomial & Taylor expansion \\
            Sonoda, Murata \cite{sonoda2017neural} & Lizorkin distribution & Ridgelet transform \\ \hline
        \end{tabular}
    \end{center}
\end{table}

Recently, some positive results have been established showing the UAP for particular deep narrow networks. Hanin and Sellke \cite{hanin2017approximating} have shown that deep narrow networks with ReLU activation function have the UAP, and require only width $n+m$. Lin and Jegelka \cite{lin2018resnet} have shown that a ResNet with ReLU activation function, arbitrary input dimension, width 1, output dimension 1 have the UAP. For activation functions other than ReLU, Kidger and Lyons \cite{kidger2020universal} have shown that deep narrow networks with any non-polynomial continuous function have the UAP, and require only width $n+m+1$. The comparison of the UAPs are shown in Table \ref{table:UAP}. It is proved in \cite{kidger2020universal,lin2018resnet} and Theorem \ref{thm:UAP-resnet} that the UAP holds as $L\rightarrow\infty$. Theorem \ref{thm:UAP-odenet} shows that the UAP holds when the layer is continuous.

\begin{table}[ht]
    \begin{center}
        \caption{A comparison of universal approximation properties}
        \label{table:UAP}
        \begin{tabular}{r|l|l|l}
            & Shallow wide NN & Deep narrow NN & ResNet \\ \hline
            References & \cite{leshno1993multilayer,sonoda2017neural} & \cite{kidger2020universal} & \cite{lin2018resnet} \\ \hline
            Input dimension $n$ & \multirow{2}{*}{$n,m$ : any} & \multirow{2}{*}{$n,m$ : any} & $n$ : any \\ 
            Output dimension $m$ & & & $m=1$ \\ \hline
            Activation function & Non-polynomial & Non-polynomial & ReLU \\ \hline
            Depth $L$ & $L=3$ & $L\to\infty$ & $L\to\infty$ \\ \hline
            Width $N$ & $N\to\infty$ & $N=n+m+1$ & $N=1$ \\ \hline
            \multicolumn{4}{l}{} \\ 
            & $(\alpha,\beta,\gamma)$-type ResNet & \multicolumn{2}{l}{$(\alpha,\beta,\gamma)$-type ODENet} \\ \hline
            References & Theorem \ref{thm:UAP-resnet} & \multicolumn{2}{l}{Theorem \ref{thm:UAP-odenet}} \\ \hline
            Input dimension $n$ & \multirow{2}{*}{$n\geq m$} & \multicolumn{2}{l}{\multirow{2}{*}{$n\geq m$}} \\
            Output dimension $m$ &  \\ \hline
            Activation function & Non-polynomial & \multicolumn{2}{l}{Non-polynomial} \\ \hline
            Depth $L$ & $L\to\infty$ & \multicolumn{2}{l}{continuous setting $(L=\infty)$} \\ \hline
            Width $N$ & $N=n+m$ & \multicolumn{2}{l}{$N=n+m$} \\ \hline
        \end{tabular}
    \end{center}
\end{table}

In this paper, we propose the $(\alpha,\beta,\gamma)$-type ODENet associated with \eqref{eq:odenet-main} whose width is $n+m$ and show the conditions for the UAP for the $(\alpha,\beta,\gamma)$-type ODENet and the $(\alpha,\beta,\gamma)$-type ResNet with \eqref{eq:resnet-main}. 
In Section \ref{sec2}, we show that the UAP holds for the $(\alpha,\beta,\gamma)$-type ODENet associated with \eqref{eq:odenet-main} and the $(\alpha,\beta,\gamma)$-type ResNet with \eqref{eq:resnet-main}. In Section \ref{sec3}, we derive the gradient of the loss function and a learning algorithm for the $(\alpha,\beta,\gamma)$-type ODENet in consideration, followed by some numerical experiments in Section \ref{sec4}. Finally, we end the paper with a conclusion in Section \ref{sec5}.

\section{Universal Approximation Theorem for $(\alpha,\beta,\gamma)$-type ODENet and $(\alpha,\beta,\gamma)$-type ResNet}\label{sec2}
\setcounter{equation}{0}

\subsection{Definition of an activation function with universal approximation property}\label{sec2.1}
Let $m$ and $n$ be natural numbers. Our main results, Theorem \ref{thm:UAP-odenet} and Theorem \ref{thm:UAP-resnet}, show that any continuous function on a compact subset on $\mathbb{R}^n$ can be approximated using the $(\alpha,\beta,\gamma)$-type ODENet and the $(\alpha,\beta,\gamma)$-type ResNet.

In this paper, the following notations are used
\begin{equation*}
    |x|:=\left(\sum_{i=1}^n|x_i|^2\right)^{\frac{1}{2}},\quad\|A\|:=\left(\sum_{i=1}^m\sum_{j=1}^n|a_{ij}|^2\right)^{\frac{1}{2}},
\end{equation*}
for any $x=(x_1,x_2,\ldots,x_n)^{\top}\in\mathbb{R}^n$ and $A=(a_{ij})_{\substack{i=1,\ldots,m \\ j=1,\ldots,n}}\in\mathbb{R}^{m\times n}$. Also, we define
\begin{equation*}
    \nabla_x^{\top}f:=\left(\frac{\partial f_i}{\partial x_j}\right)_{\substack{i=1,\ldots,m \\ j=1,\ldots,n}},\quad\nabla_xf^{\top}:=\left(\nabla_x^{\top}f\right)^{\top}
\end{equation*}
for any $f\in C^1(\mathbb{R}^n;\mathbb{R}^m)$. For a function $\sigma:\mathbb{R}\to\mathbb{R}$, we define $\mbox{\boldmath $\sigma$}:\mathbb{R}^m\to\mathbb{R}^m$ by
\begin{equation}\label{eq:activation-function}
    \mbox{\boldmath $\sigma$}(x):=\left(\begin{array}{c}
        \sigma(x_1) \\
        \sigma(x_2) \\
        \vdots \\
        \sigma(x_m)
    \end{array}\right)
\end{equation}
for $x=(x_1,x_2,\ldots,x_m)^{\top}\in\mathbb{R}^m$. For $a=(a_1,a_2,\ldots,a_m)^{\top},b=(b_1,b_2,\ldots,b_m)^{\top}\in\mathbb{R}^m$, their Hadamard product is defined by
\begin{equation}\label{eq:Hadamard-product}
    a\odot b:=\left(\begin{array}{c}
        a_1b_1 \\
        a_2b_2 \\
        \vdots \\
        a_mb_m
    \end{array}\right)\in\mathbb{R}^m.
\end{equation}

\begin{definition}[Universal approximation property for the activation function $\sigma$]\label{def:UAP-function}
    {\rm
    Let $\sigma$ be a real-valued function on $\mathbb{R}$ and $D$ be a compact subset of $\mathbb{R}^n$. Also, consider the set
    \begin{equation*}
        S:=\left\{G:D\to\mathbb{R}\left|G(\xi)=\sum_{l=1}^L\alpha_l\sigma(\mbox{\boldmath $c$}_l\cdot\xi+d_l),L\in\mathbb{N},\alpha_l,d_l\in\mathbb{R},\mbox{\boldmath $c$}_l\in\mathbb{R}^n\right.\right\}.
    \end{equation*}
    Suppose that $S$ is dense in $C(D)$. In other words, given $F\in C(D)$ and $\eta>0$, there exists a function $G\in S$ such that
    \begin{equation*}
        |G(\xi)-F(\xi)|<\eta
    \end{equation*}
    for any $\xi\in D$. Then, we say that $\sigma$ has a universal approximation property (UAP) on $D$.
    }
\end{definition}

Some activation functions with the universal approximation property are presented in Table \ref{table:UAP-function}.

\begin{table}[ht]
    \begin{center}
        \caption{Example of activation functions with universal approximation property}
        \label{table:UAP-function}
        \begin{tabular}{lll} \hline
            & Activation function & $\sigma(x)$ \\ \hline
            \multicolumn{3}{l}{\textbf{Unbounded functions}} \\
            & Truncated power function & $x_{+}^k:=\left\{\begin{array}{ll}
                x^k & x>0 \\
                0 & x\leq0
            \end{array}\right.\quad k\in\mathbb{N}\cup\{0\}$ \\
            & ReLU function & $x_{+}$ \\
            & Softplus function & $\log(1+e^x)$ \\
            \multicolumn{3}{l}{\textbf{Bounded but not integrable functions}} \\
            & Unit step function & $x_{+}^0$ \\
            & (Standard) Sigmoidal function & $(1+e^{-x})^{-1}$ \\
            & Hyperbolic tangent function & $\tanh(x)$ \\
            \multicolumn{3}{l}{\textbf{Bump functions}} \\
            & (Gaussian) Radial basis function & $\frac{1}{\sqrt{2\pi}}\exp\left(-\frac{x^2}{2}\right)$ \\
            & Dirac's $\delta$ function & $\delta(x)$ \\ \hline
        \end{tabular}
    \end{center}
\end{table}

A non-polynomial activation function in a neural network with three layers has a universal approximation property. Such result was shown by Leshno \cite{leshno1993multilayer} using functional analysis and later by Sonoda and Murata \cite{sonoda2017neural} using Ridgelet transform.

\subsection{Main Theorem for $(\alpha,\beta,\gamma)$-type ODENet}\label{sec2.2}
In this subsection, we show the universal approximation property for the $(\alpha,\beta,\gamma)$-type ODENet associated with the ODE system \eqref{eq:odenet-main}. Since the first (resp. second) equation consists of $n$ (resp. $m$) equations, the width of the $(\alpha,\beta,\gamma)$-type ODENet is $n+m$. 

\begin{definition}[$(\alpha,\beta,\gamma)$-type ODENet]\label{def:odenet}
    {\rm
    Suppose that an $m\times n$ real matrix $A$ and a function $\sigma:\mathbb{R}\to\mathbb{R}$ are given. We consider a system of ODEs
    \begin{equation}\label{eq:odenet-main}
        \left\{\begin{aligned}
            x'(t)&=\beta(t)x(t)+\gamma(t), & t\in(0,T], \\
            y'(t)&=\alpha(t)\odot\mbox{\boldmath $\sigma$}(Ax(t)), & t\in(0,T], \\
            x(0)&=\xi, & \\
            y(0)&=0, &
        \end{aligned}\right.
    \end{equation}
    where $x$ and $y$ are functions from $[0,T]$ to $\mathbb{R}^n$ and $\mathbb{R}^m$, respectively; $x(0)=\xi\in\mathbb{R}^n$ is an input data and $y(T)\in\mathbb{R}^m$ is the final output. Moreover, the functions $\alpha: [0,T] \to \mathbb{R}^m$, $\beta: [0,T] \to \mathbb{R}^{n\times n}$, and $\gamma: [0,T] \to \mathbb{R}^n$ are design parameters. The function $\mbox{\boldmath $\sigma$}:\mathbb{R}^m\to\mathbb{R}^m$ is defined by \eqref{eq:activation-function} and the operator $\odot$ denotes the Hadamard product defined by \eqref{eq:Hadamard-product}. We call $H = [\xi \mapsto y(T)]: \mathbb{R}^n\to\mathbb{R}^m$ an $(\alpha,\beta,\gamma)$-type ODENet associated with the ODE system \eqref{eq:odenet-main}.
    }
\end{definition}

For a compact subset $D\subset\mathbb{R}^n$, we define
\begin{equation*}
    S(D):=\{[\xi\mapsto y(T)]\in C(D;\mathbb{R}^m)|
        \alpha\in C^{\infty}([0,T];\mathbb{R}^m),
        \beta\in C^{\infty}([0,T];\mathbb{R}^{n\times n}),
        \gamma\in C^{\infty}([0,T];\mathbb{R}^n)\}.
\end{equation*}
We will assume that the activation function is locally Lipschitz continuous, in other words,
\begin{equation}\label{eq:locally-Lipschitz}
    \forall R>0,~\exists L_R>0~\mathrm{s.t.}\quad|\sigma(s_1)-\sigma(s_2)|\leq L_R|s_1-s_2|\quad\mathrm{for}~s_1,s_2\in[-R,R].
\end{equation}

\begin{theorem}[UAP for $(\alpha,\beta,\gamma)$-type ODENet]\label{thm:UAP-odenet}
    {\rm
    Suppose that $m \leq n$ and $\mathrm{rank}(A)=m$. If $\sigma:\mathbb{R}\to\mathbb{R}$ satisfies \eqref{eq:locally-Lipschitz} and has UAP on a compact subset $D\subset\mathbb{R}^n$, then $S(D)$ is dense in $C(D;\mathbb{R}^m)$. In other words, given $F\in C(D;\mathbb{R}^m)$ and $\eta>0$, there exists a function $H\in S(D)$ such that
    \begin{equation*}
        |H(\xi)-F(\xi)|<\eta,
    \end{equation*}
    for any $\xi\in D$.
    }
\end{theorem}

\begin{corollary}\label{cor:UAP-odenet}
    {\rm
    Let $1\leq p<\infty$. Then, $S(D)$ is dense in $L^p(D;\mathbb{R}^m)$. In other words, given $F\in L^p(D;\mathbb{R}^m)$ and $\eta>0$, there exists a function $H\in S(D)$ such that
    \begin{equation*}
        \|H-F\|_{L^p(D;\mathbb{R}^m)}<\eta.
    \end{equation*}
    }
\end{corollary}

\begin{remark}
    The assumption for $\sigma$ in Theorem \ref{thm:UAP-odenet} holds, e.g., if $\sigma$ is a non-polynomial continuous function \cite{leshno1993multilayer}.
\end{remark}

\subsection{Main Theorem for $(\alpha,\beta,\gamma)$-type ResNet}\label{sec2.3}
In this subsection, we show that a universal approximation property also holds for an $(\alpha,\beta,\gamma)$-type ResNet with the system of difference equations \eqref{eq:resnet-main}.
\begin{definition}[$(\alpha,\beta,\gamma)$-type ResNet]\label{def:resnet}
    {\rm
    Suppose that an $m\times n$ real matrix $A$ and a function $\sigma:\mathbb{R}\to\mathbb{R}$ are given. We consider a system of difference equations
    \begin{equation}\label{eq:resnet-main}
        \left\{\begin{aligned}
            x^{(l)}&=x^{(l-1)}+\beta^{(l)}x^{(l-1)}+\gamma^{(l)}, & l=1,2,\ldots,L \\
            y^{(l)}&=y^{(l-1)}+\alpha^{(l)}\odot\mbox{\boldmath $\sigma$}(Ax^{(l)}), & l=1,2,\ldots,L \\
            x^{(0)}&=\xi, & \\
            y^{(0)}&=0, &
        \end{aligned}\right.
    \end{equation}
    where $x^{(l)}$ and $y^{(l)}$ are $n$- and $m$-dimensional real vectors, for all $l=0,1,\ldots,L$, respectively. Also, $\xi\in\mathbb{R}^n$ denotes the input data while $y^{(L)}\in\mathbb{R}^m$ represents the final output. Moreover, the vectors $\alpha^{(l)}\in\mathbb{R}^m,\beta^{(l)}\in\mathbb{R}^{n\times n}$ and $\gamma\in\mathbb{R}^n~(l=1,2,\ldots,L)$ are design parameters. The functions $\mbox{\boldmath $\sigma$}:\mathbb{R}^m\to\mathbb{R}^m$ is defined by \eqref{eq:activation-function} and the operator $\odot$ denotes the Hadamard product defined by \eqref{eq:Hadamard-product}. We call the function $H=[\xi\mapsto y^{(L)}]:D\to\mathbb{R}^m$ an $(\alpha,\beta,\gamma)$-type ResNet with a system of difference equations \eqref{eq:resnet-main}.
    }
\end{definition}

For a compact subset $D\subset\mathbb{R}^n$, we define
\begin{equation*}
    S_{\mathrm{res}}(D):=\{[\xi\mapsto y^{(L)}]\in C(D;\mathbb{R}^m)|
        L\in\mathbb{N},\alpha^{(l)}\in\mathbb{R}^m,\beta^{(l)}\in\mathbb{R}^{n\times n},
        \gamma^{(l)}\in\mathbb{R}^n~(l=1,2,\ldots,L)
    \}.
\end{equation*}

\begin{theorem}[UAP for $(\alpha,\beta,\gamma)$-type ResNet]\label{thm:UAP-resnet}
    {\rm
    Suppose that $m\leq n$ and $\mathrm{rank}(A)=m$. If $\sigma:\mathbb{R}\to\mathbb{R}$ satisfies \eqref{eq:locally-Lipschitz} and has UAP on a compact subset $D\subset\mathbb{R}^n$, then $S_{\mathrm{res}}(D)$ is dense in $C(D;\mathbb{R}^m)$.
    }
\end{theorem}

\begin{remark}
    The assumption for $\sigma$ in Theorem \ref{thm:UAP-resnet} holds, e.g., if $\sigma$ is a non-polynomial continuous function \cite{leshno1993multilayer}.
\end{remark}

\subsection{Some lemmas}\label{sec2.4}
We describe some lemmas used to prove Theorems \ref{thm:UAP-odenet} and \ref{thm:UAP-resnet}.
\begin{lemma}\label{lem:linearly-independent1}
    {\rm
    Suppose that $m\leq n$. Let $\mbox{\boldmath $\sigma$}$ be a function from $\mathbb{R}^m$ to $\mathbb{R}^m$ defined by \eqref{eq:activation-function}. For any $\alpha,d\in\mathbb{R}^m$ and $C=(\mbox{\boldmath $c$}_1,\mbox{\boldmath $c$}_2,\ldots,\mbox{\boldmath $c$}_m)^{\top}\in\mathbb{R}^{m\times n}$ which has no zero rows (i.e. $\mbox{\boldmath $c$}_l\neq 0$ for $l=1,2,\ldots,m$), there exist $\tilde{\alpha}^{(l)},\tilde{d}^{(l)}\in\mathbb{R}^m$, and $\tilde{C}^{(l)}\in\mathbb{R}^{m\times n}~(l=1,2,\ldots,m)$ such that
    \begin{equation*}
        \alpha\odot\mbox{\boldmath $\sigma$}(C\xi+d)=\sum_{l=1}^m\tilde{\alpha}^{(l)}\odot\mbox{\boldmath $\sigma$}(\tilde{C}^{(l)}\xi+\tilde{d}^{(l)}),
    \end{equation*}
    for any $\xi\in\mathbb{R}^n$, and $\mathrm{rank}(\tilde{C}^{(l)})=m$, for all $l=1,2,\ldots,m$. Moreover, if $m=n$, we can choose $\tilde{C}^{(l)}\in\mathbb{R}^{n\times n}$ such that $\det\tilde{C}^{(l)}>0$, for all $l=1,2,\ldots,n$.
    }
\end{lemma}

\begin{proof}
    Let $m\leq n$. For all $l=1,2,\ldots,m$, there exists $\tilde{C}^{(l)}=(\tilde{\mbox{\boldmath $c$}}_1^{(l)},\tilde{\mbox{\boldmath $c$}}_2^{(l)},\ldots,\tilde{\mbox{\boldmath $c$}}_m^{(l)})^{\top}\in\mathbb{R}^{m\times n}$ such that $\tilde{\mbox{\boldmath $c$}}_l^{(l)}=\mbox{\boldmath $c$}_l$, $\mathrm{rank}(\tilde{C}^{(l)})=m$. Then, we put
    \begin{equation*}
        \tilde{\alpha}_k^{(l)}:=\left\{\begin{array}{ll}
            \alpha_k, & \mathrm{if}~l=k, \\
            0, & \mathrm{if}~l\neq k,
        \end{array}\right.\quad \tilde{d}_k^{(l)}:=\left\{\begin{array}{ll}
            d_k, & \mathrm{if}~l=k, \\
            0, & \mathrm{if}~l\neq k.
        \end{array}\right.
    \end{equation*}
    Looking at the $k$-th component, we see that for any $\xi\in\mathbb{R}^n$, we have
    \begin{equation*}
        \sum_{l=1}^m\tilde{\alpha}_k^{(l)}\sigma(\tilde{\mbox{\boldmath $c$}}_k^{(l)}\cdot\xi+\tilde{d}_k^{(l)})=\tilde{\alpha}_k^{(k)}\sigma(\tilde{\mbox{\boldmath $c$}}_k^{(k)}\cdot\xi+\tilde{d}_k^{(k)})=\alpha_k\sigma(\mbox{\boldmath $c$}_k\cdot\xi+d_k).
    \end{equation*}
    Therefore,
    \begin{equation*}
        \sum_{l=1}^m\tilde{\alpha}^{(l)}\odot\mbox{\boldmath $\sigma$}(\tilde{C}^{(l)}\xi+\tilde{d}^{(l)})=\alpha\odot\mbox{\boldmath $\sigma$}(C\xi+d).
    \end{equation*}
    Now, if $m=n$, then $\mathrm{rank}(\tilde{C}^{(l)})=n$, and so $\det(\tilde{C}^{(l)})\neq0$. In particular, we can choose $\tilde{C}^{(l)}$ such that $\det(\tilde{C}^{(l)})>0$.
\end{proof}

\begin{lemma}\label{lem:linearly-independent2}
    {\rm
    Suppose that $m\leq n$. Let $\mbox{\boldmath $\sigma$}$ be a function from $\mathbb{R}^m$ to $\mathbb{R}^m$. For any $L\in\mathbb{N},\alpha^{(l)},d^{(l)}\in\mathbb{R}^m,C^{(l)}\in\mathbb{R}^{m\times n}~(l=1,2,\ldots,L)$, there exists $L'\in\mathbb{N},\tilde{\alpha}^{(l)},\tilde{d}^{(l)}\in\mathbb{R}^m,\tilde{C}^{(l)}\in\mathbb{R}^{m\times n}~(l=1,2,\ldots,L')$ such that
    \begin{equation*}
        \frac{1}{L}\sum_{l=1}^L\alpha^{(l)}\odot\mbox{\boldmath $\sigma$}(C^{(l)}\xi+d^{(l)})=\frac{1}{L'}\sum_{l=1}^{L'}\tilde{\alpha}^{(l)}\odot\mbox{\boldmath $\sigma$}(\tilde{C}^{(l)}\xi+\tilde{d}^{(l)})
    \end{equation*}
    for any $\xi\in\mathbb{R}^n$, and $\mathrm{rank}(\tilde{C}^{(l)})=m$, for all $l=1,2,\ldots,L'$. Moreover, if $m=n$, we can choose $\tilde{C}^{(l)}\in\mathbb{R}^{m\times n}$ such that $\det\tilde{C}^{(l)}>0$, for all $l=1,2,\ldots,L'$.
    }
\end{lemma}

\begin{proof}
    This follows from Lemma \ref{lem:linearly-independent1}.
\end{proof}

\begin{lemma}\label{lem:positive-definite}
    {\rm
    Suppose that $m<n$. Let $A$ be an $m\times n$ real matrix satisfying $\mathrm{rank}(A)=m$. Then, for any $C\in\mathbb{R}^{m\times n}$ satisfying $\mathrm{rank}(C)=m$, there exists $P\in\mathbb{R}^{n\times n}$ such that
    \begin{equation}\label{eq:positive-definite}
        C=AP,\quad\det P>0.
    \end{equation}
    In addition, if $m=n$ and $\mathrm{sgn}(\det C)=\mathrm{sgn}(\det A)$, there exists $P\in\mathbb{R}^{n\times n}$ such that \eqref{eq:positive-definite}.
    }
\end{lemma}

\begin{proof}
    \begin{enumerate}[(i)]
        \item Suppose that $m<n$. From $\mathrm{rank}(A)=\mathrm{rank}(C)=m$, there exists $\bar{A},\bar{C}\in\mathbb{R}^{(n-m)\times n}$ such that
        \begin{equation*}
            \det\tilde{A}>0,\quad\tilde{A}=\left(\begin{array}{c}
                A \\
                \bar{A}
            \end{array}\right),\quad \det\tilde{C}>0,\quad\tilde{C}=\left(\begin{array}{c}
                C \\
                \bar{C}
            \end{array}\right).
        \end{equation*}
        If we put $P:=\tilde{A}^{-1}\tilde{C}$, we get $\det P>0$, $C=AP$.
        \item Suppose that $m=n$. We put $P:=A^{-1}C$. Because $\mathrm{sgn}(\det C)=\mathrm{sgn}(\det A)$, we have $\det P>0$, and so $C=AP$.
    \end{enumerate}
\end{proof}

\begin{lemma}\label{lem:positive-det-continuous}
    {\rm
    Let $p\in[1,\infty)$. Suppose that
    \begin{equation*}
        P(t)=P^{(l)}\in\mathbb{R}^{n\times n},\quad \det P^{(l)}>0,
    \end{equation*}
    for $t_{l-1}\leq t<t_l$, and for all $l=1,2,\ldots,L$, where $t_0=0$ and $t_L=T$. Then, there exists a real number $C>0$ such that, for any $\varepsilon>0$, there exists $P^{\varepsilon}\in C([0,T];\mathbb{R}^{n\times n})$ such that
    \begin{equation*}
        \|P^{\varepsilon}-P\|_{L^p(0,T;\mathbb{R}^{n\times n})}<\varepsilon,\quad\det P^{\varepsilon}(t)>0,\quad\mathrm{and}\quad\|P^{\varepsilon}(t)\|\leq C,
    \end{equation*}
    for any $t\in[0,T]$.
    }
\end{lemma}

\begin{proof}
    We define $\mathrm{GL}^{+}(n,\mathbb{R}):=\{A\in\mathbb{R}^{n\times n}|\det A>0\}$. From \cite[Chapter 9, p.239]{baker2003matrix}, $\mathrm{GL}^{+}(n,\mathbb{R})$ is path-connected. For all $l=1,2,\ldots,L$, there exists $Q^{(l)}\in C([0,1];\mathbb{R}^{n\times n})$ such that
    \begin{equation*}
        Q^{(l)}(0)=P^{(l)},\quad Q^{(l)}(1)=P^{(l+1)},\quad\mathrm{and}\quad\det Q^{(l)}(s)>0,
    \end{equation*}
    for any $s\in[0,1]$. For $\delta>0$, we put
    \begin{equation*}
        Q^{\delta}(t):=\left\{\begin{array}{lll}
            P^{(1)}, & -\infty<t<t_1, & \\
            \displaystyle{Q^{(l)}\left(\frac{t-t_l}{\delta}\right)}, & t_l\leq t<t_l+\delta, & (l=1,2,\ldots,L-1), \\
            P^{(l)} & t_{l-1}+\delta\leq t<t_l, & (l=2,3,\ldots,L-2), \\
            P^{(L)} & t_{L-1}+\delta\leq t<\infty. &
        \end{array}\right.
    \end{equation*}
    Then, $Q^{\delta}$ is a continuous function from $\mathbb{R}$ to $\mathbb{R}^{n\times n}$. There exists a $C_0>0$ such that $\det Q^{\delta}(t)\geq C_0$, for any $t\in\mathbb{R}$. Let $\{\varphi_{\varepsilon}\}_{\varepsilon>0}$ be a sequence of Friedrichs' mollifiers in $\mathbb{R}$. We put
    \begin{equation*}
        P^{\varepsilon}(t):=(\varphi_{\varepsilon}*Q^{\delta})(t).
    \end{equation*}
    Then, $P^{\varepsilon}\in C^{\infty}(\mathbb{R};\mathbb{R}^{n\times n})$. Since
    \begin{equation*}
        \lim_{\varepsilon\to0}\|P^{\varepsilon}-Q^{\delta}\|_{C([0,T];\mathbb{R}^{n\times n})}=0,
    \end{equation*}
    there exists a number $\varepsilon_0>0$ such that, for any $\varepsilon\leq\varepsilon_0$,
    \begin{equation*}
        \det P^{\varepsilon}(t)\geq\frac{C_0}{2}
    \end{equation*}
    for all $t\in[0,T]$. Because $Q^{\delta}$ is bounded, there exists a number $C>0$ such that $\|P^{\varepsilon}(t)\|\leq C$, for any $t\in[0,T]$. Now, we note that
    \begin{equation*}
        \|P^{\varepsilon}-P\|_{L^p(0,T;\mathbb{R}^{n\times n})}\leq\|P^{\varepsilon}-Q^{\delta}\|_{L^p(0,T;\mathbb{R}^{n\times n})}+\|Q^{\delta}-P\|_{L^p(0,T;\mathbb{R}^{n\times n})}.
    \end{equation*}
    The last summand is calculated as follows
    \begin{align*}
        \|Q^{\delta}-P\|_{L^p(0,T;\mathbb{R}^{n\times n})}^p &= \int_0^T\|Q^{\delta}(t)-P(t)\|^pdt, \\
        &= \sum_{l=1}^{L-1}\int_{t_l}^{t_l+\delta}\left\|Q^{(l)}\left(\frac{t-t_l}{\delta}\right)-P^{(l+1)}\right\|^pdt, \\
        &= \delta\sum_{l=1}^{L-1}\int_0^1\|Q^{\delta}(s)-P^{(l+1)}\|^pds.
    \end{align*}
    Hence, if $\delta\to0$, then $\|Q^{\delta}-P\|_{L^p(0,T;\mathbb{R}^{n\times n})}\to0$. Therefore,
    \begin{equation*}
        \|P^{\varepsilon}-P\|_{L^p(0,T;\mathbb{R}^{n\times n})}<\varepsilon,
    \end{equation*}
    for any $\varepsilon>0$.
\end{proof}

\begin{remark}
    Lemma \ref{lem:positive-det-continuous} does not hold when $p=\infty$ (except when $P$ is a constant function) because the uniform limit of continuous functions is also continuous.
\end{remark}

\subsection{Proofs}\label{sec2.5}
In this subsection, we provide the proof of Theorem \ref{thm:UAP-odenet} and Theorem \ref{thm:UAP-resnet}.

\subsubsection{Proof of Theorem \ref{thm:UAP-odenet}}\label{sec2.5.1}
\begin{proof}
    Since $\mbox{\boldmath $\sigma$}\in C(\mathbb{R}^m;\mathbb{R}^m)$ is defined by \eqref{eq:activation-function}, where $\sigma\in C(\mathbb{R})$ satisfies a UAP, then given $F\in C(D;\mathbb{R}^m)$ and $\eta>0$, there exist a positive integer $L$, $\mathbb{R}^m$-valued vectors $\alpha^{(l)}$ and $d^{(l)}$, and matrices $C^{(l)}\in\mathbb{R}^{m\times n}$, for all $l=1,2,\ldots,L$, such that
    \begin{equation*}
        G(\xi)=\frac{T}{L}\sum_{l=1}^L\alpha^{(l)}\odot\mbox{\boldmath $\sigma$}(C^{(l)}\xi+d^{(l)}),
    \end{equation*}
    \begin{equation}\label{eq:eva-UAP}
        |G(\xi)-F(\xi)|<\frac{\eta}{2},
    \end{equation}
    for any $\xi\in D$. From Lemma \ref{lem:linearly-independent2}, we know that $\mathrm{rank}(C^{(l)})=m$, for $l=1,2,\ldots,L$. In addition, when $m=n$, we have $\mathrm{sgn}(\det A)=\mathrm{sgn}(\det C^{(l)})$. In view of Lemma \ref{lem:positive-definite}, there exists a matrix $P^{(l)}\in\mathbb{R}^{n\times n}$ such that $\det P^{(l)}>0$ and $C^{(l)}=AP^{(l)}$, for each $l=1,2,\ldots,L$. We put $q^{(l)}:=A^{\top}(AA^{\top})^{-1}d^{(l)}$ so that $d^{(l)}=Aq^{(l)}$. In addition, we let
    \begin{equation*}
        \alpha(t):=\alpha^{(l)},\quad P(t):=P^{(l)},\quad q(t):=q^{(l)},\quad\frac{l-1}{L}T\leq t<\frac{l}{L}T.
    \end{equation*}
    Then, $\det P(t)>0$ for any $t\in[0,T]$ and
    \begin{equation*}
        G(\xi)=\frac{T}{L}\sum_{l=1}^L\alpha^{(l)}\odot\mbox{\boldmath $\sigma$}(AP^{(l)}\xi+Aq^{(l)})=\int_0^T\alpha(t)\odot\mbox{\boldmath $\sigma$}(A(P(t)\xi+q(t)))dt.
    \end{equation*}
    Let $\{\varphi_{\varepsilon}\}_{\varepsilon>0}$ be a sequence of Friedrichs' mollifiers. We put $\alpha^{\varepsilon}(t):=(\varphi_{\varepsilon}*\alpha)(t)$ and $q^{\varepsilon}(t):=(\varphi_{\varepsilon}*q)(t)$. Then, $\alpha^{\varepsilon}\in C^{\infty}([0,T];\mathbb{R}^m)$ and $q^{\varepsilon}\in C^{\infty}([0,T];\mathbb{R}^n)$. From Lemma \ref{lem:positive-det-continuous}, there exists a real number $C>0$ such that, given $\eta>0$, there exists $P^{\varepsilon}\in C^{\infty}([0,T];\mathbb{R}^{n\times n})$ from which we have
    \begin{equation*}
        \|P^{\varepsilon}-P\|_{L^1(0,T;\mathbb{R}^{n\times n})}<\eta,\quad\det P^{\varepsilon}(t)>0,\quad\|P^{\varepsilon}(t)\|\leq C,
    \end{equation*}
    for any $t\in[0,T]$. If we put
    \begin{equation}\label{eq:proof-x}
        x^{\varepsilon}(t;\xi):=P^{\varepsilon}(t)\xi+q^{\varepsilon}(t),
    \end{equation}
    \begin{equation}\label{eq:proof-y}
        y^{\varepsilon}(t;\xi):=\int_0^T\alpha^{\varepsilon}(s)\odot\mbox{\boldmath $\sigma$}(Ax^{\varepsilon}(s;\xi))ds,
    \end{equation}
    then
    \begin{equation*}
        y^{\varepsilon}(T;\xi)=\int_0^T\alpha^{\varepsilon}(t)\odot\mbox{\boldmath $\sigma$}(A(P^{\varepsilon}(t)\xi+q^{\varepsilon}(t)))dt.
    \end{equation*}
    Hence, we have
    \begin{alignat*}{2}
        &|y^{\varepsilon}(T;\xi)-G(\xi)|
        &\leq &\int_0^T\left|\alpha^{\varepsilon}(t)\odot\mbox{\boldmath $\sigma$}(A(P^{\varepsilon}(t)\xi+q^{\varepsilon}(t)))-\alpha(t)\odot\mbox{\boldmath $\sigma$}(A(P(t)\xi+q(t)))\right|dt, \\
        & &\leq &\int_0^T|\alpha^{\varepsilon}(t)-\alpha(t)||\mbox{\boldmath $\sigma$}(A(P(t)\xi+q(t)))|dt, \\
        & & &+ \int_0^T|\alpha^{\varepsilon}(t)||\mbox{\boldmath $\sigma$}(A(P^{\varepsilon}(t)\xi+q^{\varepsilon}(t)))-\mbox{\boldmath $\sigma$}(A(P(t)\xi+q(t)))|dt.
    \end{alignat*}
    Because $P$ and $q$ are piecewise constant functions, then they are bounded. Since $\mbox{\boldmath $\sigma$}\in C(\mathbb{R}^m;\mathbb{R}^m)$, there exists $M>0$ such that $|\mbox{\boldmath $\sigma$}(A(P(t)\xi+q(t)))|\leq M$, for any $t\in[0,T]$. On the other had, we have the estimate
    \begin{equation*}
        |\alpha^{\varepsilon}(t)|\leq\int_{\mathbb{R}}\varphi_{\varepsilon}(t-s)|\alpha(s)|ds\leq\|\alpha\|_{L^{\infty}(0,T;\mathbb{R}^m)}\int_{\mathbb{R}}\varphi_{\varepsilon}(\tau)d\tau=\|\alpha\|_{L^{\infty}(0,T;\mathbb{R}^m)}.
    \end{equation*}
    Similarly, because $\|q^{\varepsilon}\|_{L^{\infty}(0,T;\mathbb{R}^n)}\leq\|q\|_{L^{\infty}(0,T;\mathbb{R}^n)}$, then $q^{\varepsilon}$ is bounded. We assume that $A(P^{\varepsilon}(t)\xi+q^{\varepsilon}(t))$, $A(P(t)\xi+q(t))\in [-R, R]^m$, for any $t\in[0,T]$,
    \begin{align*}
        &|\mbox{\boldmath $\sigma$}(A(P^{\varepsilon}(t)\xi+q^{\varepsilon}(t)))-\mbox{\boldmath $\sigma$}(A(P(t)\xi+q(t)))| \\
        &\leq L_R\|A\|\left(\|P^{\varepsilon}(t)-P(t)\|(\max_{\xi\in D}|\xi|)+|q^{\varepsilon}(t)-q(t)|\right).
    \end{align*}
    Therefore,
    \begin{align*}
        &|y^{\varepsilon}(T;\xi)-G(\xi)|\leq M\|\alpha^{\varepsilon}-\alpha\|_{L^1(0,T;\mathbb{R}^m)} \\
        &+ L_R\|A\|\|\alpha\|_{L^{\infty}(0,T;\mathbb{R}^m)}\left(\|P^{\varepsilon}-P\|_{L^1(0,T;\mathbb{R}^{n\times n})}(\max_{\xi\in D}|\xi|)+\|q^{\varepsilon}-q\|_{L^1(0,T;\mathbb{R}^n)}\right).
    \end{align*}
    We know that there exists a number $\varepsilon>0$ such that
    \begin{equation}\label{eq:eva-mollifier}
        |y^{\varepsilon}(T;\xi)-G(\xi)|<\frac{\eta}{2},
    \end{equation}
    for any $\xi\in D$. Thus, from \eqref{eq:eva-UAP} and \eqref{eq:eva-mollifier},
    \begin{equation*}
        |y^{\varepsilon}(T;\xi)-F(\xi)|\leq|y^{\varepsilon}(T;\xi)-G(\xi)|+|G(\xi)-F(\xi)|<\eta,
    \end{equation*}
    for any $\xi\in D$. For all $t\in[0,T]$, we know that $\det P^{\varepsilon}(t)>0$, so $P^{\varepsilon}(t)$ is invertible. This allows us to define
    \begin{equation*}
        \beta(t):=\left(\frac{d}{dt}P^{\varepsilon}(t)\right)\left(P^{\varepsilon}(t)\right)^{-1},\quad\gamma(t):=\frac{d}{dt}q^{\varepsilon}(t)-\beta(t)q^{\varepsilon}(t).
    \end{equation*}
    This gives us
    \begin{equation*}
        \frac{d}{dt}P^{\varepsilon}(t)=\beta(t)P^{\varepsilon}(t),\quad\frac{d}{dt}q^{\varepsilon}(t)=\beta(t)q^{\varepsilon}(t)+\gamma(t).
    \end{equation*}
    In view of \eqref{eq:proof-x} and \eqref{eq:proof-y},
    \begin{equation*}
        \frac{d}{dt}x^{\varepsilon}(t;\xi)=\frac{d}{dt}P^{\varepsilon}(t)\xi+\frac{d}{dt}q^{\varepsilon}(t)=\beta(t)P^{\varepsilon}(t)\xi+\beta(t)q^{\varepsilon}(t)+\gamma(t)=\beta(t)x^{\varepsilon}(t;\xi)+\gamma(t),
    \end{equation*}
    \begin{equation*}
        \frac{d}{dt}y^{\varepsilon}(t;\xi)=\alpha^{\varepsilon}(t)\odot\mbox{\boldmath $\sigma$}(Ax^{\varepsilon}(t;\xi)).
    \end{equation*}
    Hence, $y^{\varepsilon}(T,\cdot)\in S(D)$. Therefore, given $F\in C(D;\mathbb{R}^m)$ and $\eta>0$, there exist some functions $\alpha\in C^{\infty}([0,T];\mathbb{R}^m)$, $\beta\in C^{\infty}([0,T];\mathbb{R}^{n\times n})$, and $\gamma\in C^{\infty}([0,T];\mathbb{R}^n)$ such that
    \begin{equation*}
        |y(T;\xi)-F(\xi)|<\eta,
    \end{equation*}
    for any $\xi\in D$.
\end{proof}

\subsubsection{Proof of Theorem \ref{thm:UAP-resnet}}\label{sec2.5.2}
\begin{proof}
    Again, we start with the fact that $\mbox{\boldmath $\sigma$}\in C(\mathbb{R}^m;\mathbb{R}^m)$ is defined by \eqref{eq:activation-function}, where $\sigma\in C(\mathbb{R})$ satisfies a UAP; that is, given $F\in C(D;\mathbb{R}^m)$ and $\eta>0$, there exist a positive integer $L$, $\mathbb{R}^m$-valued vectors $\alpha^{(l)}$ and $d^{(l)}$, and matrices $C^{(l)}\in\mathbb{R}^{m\times n}$, for all $l=1,2,\ldots,L$, such that
    \begin{equation*}
        G(\xi)=\sum_{l=1}^L\alpha^{(l)}\odot\mbox{\boldmath $\sigma$}(C^{(l)}\xi+d^{(l)}),
    \end{equation*}
    \begin{equation*}
        |G(\xi)-F(\xi)|<\eta,
    \end{equation*}
    for any $\xi\in D$. By virtue of Lemma \ref{lem:linearly-independent2}, we know that $\mathrm{rank}(C^{(l)})=m$, for all $l=1,2,\ldots,L$. Moreover, if $m=n$, we have $\mathrm{sgn}(\det A)=\mathrm{sgn}(\det C^{(l)})$. On the other hand, from Lemma \ref{lem:positive-definite}, there exists $P^{(l)}\in\mathbb{R}^{n\times n}$ such that $\det P^{(l)}>0$ and $C^{(l)}=AP^{(l)}$, for each $l=1,2,\ldots,L$. Putting $q^{(l)}:=A^{\top}(AA^{\top})^{-1}d^{(l)}$, we get $d^{(l)}=Aq^{(l)}$, from which we obtain
    \begin{equation*}
        G(\xi)=\sum_{l=1}^L\alpha^{(l)}\odot\mbox{\boldmath $\sigma$}(A(P^{(l)}\xi+q^{(l)})).
    \end{equation*}
    Next, we define
    \begin{equation*}
        x^{(l)}:=P^{(l)}\xi+q^{(l)},\quad y^{(l)}:=\sum_{i=1}^l\alpha^{(i)}\odot\mbox{\boldmath $\sigma$}(Ax^{(i)}),
    \end{equation*}
    \begin{equation*}
        \beta^{(l)}:=(P^{(l)}-P^{(l-1)})(P^{(l-1)})^{-1},\quad\gamma^{(l)}:=q^{(l)}-q^{(l-1)}-\beta^{(l)}q^{(l-1)},
    \end{equation*}
    for all $l=1,2,\ldots,L$, and set $P^{(0)}:=I_n$, $q^{(0)}=0$. Because $P^{(l)}-P^{(l-1)}=\beta^{(l)}P^{(l-1)}$ and $q^{(l)}-q^{(l-1)}=\beta^{(l)}q^{(l-1)}+\gamma^{(l)}$ hold true, then
    \begin{equation*}
        x^{(l)}-x^{(l-1)}=(P^{(l)}-P^{(l-1)})\xi+(q^{(l)}-q^{(l-1)})=\beta^{(l)}x^{(l-1)}+\gamma^{(l)},
    \end{equation*}
    \begin{equation*}
        y^{(L)}=\sum_{l=1}^L\alpha^{(l)}\odot\mbox{\boldmath $\sigma$}(A(P^{(l)}\xi+q^{(l)}))=G(\xi).
    \end{equation*}
    Hence, $[\xi\mapsto y^{(L)}]\in S_{\mathrm{res}}(D)$. Therefore, given $F\in C(D;\mathbb{R}^m)$ and $\eta>0$, there exists $L\in\mathbb{N},\alpha^{(l)}\in\mathbb{R}^m,\beta^{(l)}\in\mathbb{R}^{n\times n},\gamma^{(l)}\in\mathbb{R}^n~(l=1,2,\ldots,L)$ such that
    \begin{equation*}
        |y^{(L)}-F(\xi)|<\eta,
    \end{equation*}
    for any $\xi\in D$.
\end{proof}

\section{The gradient and learning algorithm}\label{sec3}
\setcounter{equation}{0}

\subsection{The gradient of the loss function with respect to the design parameter}\label{sec3.1}
We consider the $(\alpha,\beta,\gamma)$-type ODENet associated with the ODE system of \eqref{eq:odenet-main}. We also consider the approximation of $F\in C(D;\mathbb{R}^m)$. Let $K\in\mathbb{N}$ be the number of training data and $\{(\xi^{(k)},F(\xi^{(k)}))\}_{k=1}^K\subset D\times\mathbb{R}^m$ be the training data. We divide the label of the training data into the following disjoint sets.
\begin{equation*}
    \{1,2,\ldots,K\}=I_1\cup I_2\cup\cdots\cup I_M~(\mathrm{disjoint})\quad(1\leq M\leq K)
\end{equation*}
Let $x^{(k)}(t)$ and $y^{(k)}(t)$ be the solution to \eqref{eq:odenet-main} with the initial value $\xi^{(k)}$. For all $\mu=1,2,\ldots,M$, let $\mbox{\boldmath $x$}=(x^{(k)})_{k\in I_{\mu}}$ and $\mbox{\boldmath $y$}=(y^{(k)})_{k\in I_{\mu}}$. We define the loss function as follows:
\begin{equation}\label{eq:loss-function-minibatch}
    e_{\mu}[\mbox{\boldmath $x$},\mbox{\boldmath $y$}]=\frac{1}{|I_{\mu}|}\sum_{k\in I_{\mu}}\left|y^{(k)}(T)-F(\xi^{(k)})\right|^2,
\end{equation}
\begin{equation}\label{eq:loss-function}
    E=\frac{1}{K}\sum_{k=1}^K\left|y^{(k)}(T)-F(\xi^{(k)})\right|^2.
\end{equation}
We consider the learning for each label set using the gradient method. We fix $\mu=1,2,\ldots,M$. Let $\lambda^{(k)}:[0,T]\to\mathbb{R}^n$ be the adjoint and satisfy the following adjoint equation for any $k\in I_{\mu}$.
\begin{equation}\label{eq:adjoint-equation}
    \left\{\begin{aligned}
        \frac{d}{dt}\lambda^{(k)}(t)&=-(\beta(t))^{\top}\lambda^{(k)}(t)-\frac{1}{|I_{\mu}|}A^{\top}\left(\left(y^{(k)}(T)-F(\xi^{(k)})\right)\odot\alpha(t)\odot\mbox{\boldmath $\sigma$}'(Ax^{(k)}(t))\right), \\
        \lambda^{(k)}(T)&=0.
    \end{aligned}\right.
\end{equation}
Then, the gradient $G[\alpha]^{(\mu)}\in C([0,T];\mathbb{R}^m),G[\beta]^{(\mu)}\in C([0,T];\mathbb{R}^{n\times n})$ and $G[\gamma]^{(\mu)}\in C([0,T];\mathbb{R}^n)$ of the loss function \eqref{eq:loss-function-minibatch} at $\alpha\in C([0,T];\mathbb{R}^m),\beta\in C([0,T];\mathbb{R}^{n\times n})$ and $\gamma\in C([0,T];\mathbb{R}^n)$ with respect to $L^2(0,T;\mathbb{R}^m),L^2(0,T;\mathbb{R}^{n\times n}),L^2(0,T;\mathbb{R}^n)$ can be represented as
\begin{equation*}
    G[\alpha]^{(\mu)}(t)=\frac{1}{|I_{\mu}|}\sum_{k\in I_{\mu}}\left(y^{(k)}(T)-F(\xi^{(k)})\right)\odot\mbox{\boldmath $\sigma$}(Ax^{(k)}(t)),
\end{equation*}
\begin{equation*}
    G[\beta]^{(\mu)}(t)=\sum_{k\in I_{\mu}}\lambda^{(k)}(t)\left(x^{(k)}(t)\right)^{\top},\quad  G[\gamma]^{(\mu)}(t)=\sum_{k\in I_{\mu}}\lambda^{(k)}(t),
\end{equation*}
respectively.

\subsection{Learning algorithm}\label{sec3.2}
In this section, we describe the learning algorithm of the $(\alpha,\beta,\gamma)$-type ODENet associated with an ODE system \eqref{eq:odenet-main}. The initial value problems of ordinary differential equations \eqref{eq:odenet-main} and \eqref{eq:adjoint-equation} are computed using the explicit Euler method. Let $h$ be the size of the time step. We define $L:=\lfloor T/h\rfloor$. By discretizing the ordinary differential equations \eqref{eq:odenet-main}, we obtain
\begin{equation*}
    \left\{\begin{aligned}
        \frac{x_{l+1}^{(k)}-x_l^{(k)}}{h}&=\beta_lx_l^{(k)}+\gamma_l, & l=0,1,\ldots,L-1, \\
        \frac{y_{l+1}^{(k)}-y_l^{(k)}}{h}&=\alpha_l\odot\mbox{\boldmath $\sigma$}(Ax_l^{(k)}), & l=0,1,\ldots,L-1, \\
        x_0^{(k)}&=\xi^{(k)}, & \\
        y_0^{(k)}&=0, &
    \end{aligned}\right.
\end{equation*}
for any $k\in I_{\mu}$. Furthermore, by discretizing the adjoint equation \eqref{eq:adjoint-equation}, we obtain
\begin{equation*}
    \left\{\begin{aligned}
        \frac{\lambda_l^{(k)}-\lambda_{l-1}^{(k)}}{h}&=-\beta_l^{\top}\lambda_l^{(k)}-\frac{1}{|I_{\mu}|}A^{\top}\left(\left(y_L^{(k)}-F(\xi^{(k)})\right)\odot\alpha_l\odot\mbox{\boldmath $\sigma$}'(Ax_l^{(k)})\right), \\
        \lambda_L^{(k)}&=0,
    \end{aligned}\right.
\end{equation*}
with $l=L,L-1,\ldots,1$ for any $k\in I_{\mu}$. Here we put
\begin{equation*}
    \alpha_l=\alpha(lh),\quad\beta_l=\beta(lh),\quad\gamma_l=\gamma(lh),
\end{equation*}
for all $l=0,1,\ldots,L$.

We perform the optimization of the loss function \eqref{eq:loss-function} using stochastic gradient descent (SGD). We show the learning algorithm in Algorithm \ref{alg:SGD}.

\begin{algorithm}
    \caption{Stochastic gradient descent method for $(\alpha,\beta,\gamma)$-type ODENet}
    \label{alg:SGD}
    \begin{algorithmic}[1]
        \STATE Choose $\eta>0$ and $\tau>0$
        \STATE Set $\nu=0$ and choose $\alpha_{(0)}\in\prod_{l=0}^L\mathbb{R}^m,\beta_{(0)}\in\prod_{l=0}^L\mathbb{R}^{n\times n}$, $\gamma_{(0)}\in\prod_{l=0}^L\mathbb{R}^n$ and (fixed) $A$
        \REPEAT
            \STATE Divide the label of the training data $\{(\xi^{(k)},F(\xi^{(k)}))\}_{k=1}^K$ into the following disjoint sets
            \begin{equation*}
                \{1,2,\ldots,K\}=I_1\cup I_2\cup\cdots\cup I_M~(\mathrm{disjoint}),\quad(1\leq M\leq K)
            \end{equation*}
            \STATE Set $\alpha^{(1)}=\alpha_{(\nu)},\beta^{(1)}=\beta_{(\nu)}$ and $\gamma^{(1)}=\gamma_{(\nu)}$
            \FOR{$\mu=1,M$}
                \STATE Solve
                \begin{equation*}
                    \left\{\begin{aligned}
                        \frac{x_{l+1}^{(k)}-x_l^{(k)}}{h}&=\beta_lx_l^{(k)}+\gamma_l, & l=0,1,\ldots,L-1, \\
                        \frac{y_{l+1}^{(k)}-y_l^{(k)}}{h}&=\alpha_l\odot\mbox{\boldmath $\sigma$}(Ax_l^{(k)}), & l=0,1,\ldots,L-1, \\
                        x_0^{(k)}&=\xi^{(k)}, & \\
                        y_0^{(k)}&=0, &
                    \end{aligned}\right.
                \end{equation*}
                for any $k\in I_{\mu}$
                \STATE Solve
                \begin{equation*}
                    \left\{\begin{aligned}
                        \frac{\lambda_l^{(k)}-\lambda_{l-1}^{(k)}}{h}&=-\beta_l^{\top}\lambda_l^{(k)}-\frac{1}{|I_{\mu}|}A^{\top}\left(\left(y_L^{(k)}-F(\xi^{(k)})\right)\odot\alpha_l\odot\mbox{\boldmath $\sigma$}'(Ax_l^{(k)})\right), \\
                        \lambda_L^{(k)}&=0,
                    \end{aligned}\right.
                \end{equation*}
                with $l=L,L-1,\ldots,1$ for any $k\in I_{\mu}$
                \STATE Compute the gradients
                \begin{equation*}
                    G[\alpha]_l^{(\mu)}=\frac{1}{|I_{\mu}|}\sum_{k\in I_{\mu}}\left(y_L^{(k)}-F(\xi^{(k)})\right)\odot\mbox{\boldmath $\sigma$}(Ax_l^{(k)}),
                \end{equation*}
                \begin{equation*}
                    G[\beta]_l^{(\mu)}=\sum_{k\in I_{\mu}}\lambda_l^{(k)}(x_l^{(k)})^{\top},\quad G[\gamma]_l^{(\mu)}=\sum_{k\in I_{\mu}}\lambda_l^{(k)}
                \end{equation*}
                \STATE Set 
                \begin{equation*}
                    \alpha_l^{(\mu+1)}=\alpha_l^{(\mu)}-\tau G[\alpha]_l^{(\mu)},\quad\beta_l^{(\mu+1)}=\beta_l^{(\mu)}-\tau G[\beta]_l^{(\mu)},
                \end{equation*}
                \begin{equation*}
                    \gamma_l^{(\mu+1)}=\gamma_l^{(\mu)}-\tau G[\gamma]_l^{(\mu)}
                \end{equation*} \label{alg:enum}
            \ENDFOR
            \STATE Set $\alpha_{(\nu+1)}=(\alpha_l^{(M)})_{l=0}^L,\beta_{(\nu+1)}=(\beta_l^{(M)})_{l=0}^L$ and $\gamma_{(\nu+1)}=(\gamma_l^{(M)})_{l=0}^L$
            \STATE Shuffle the training data $\{(\xi^{(k)},F(\xi^{(k)}))\}_{k=1}^K$ randomly and set $\nu=\nu+1$
        \UNTIL{$\max(\|\alpha_{(\nu)}-\alpha_{(\nu-1)}\|,\|\beta_{(\nu)}-\beta_{(\nu-1)}\|,\|\gamma_{(\nu)}-\gamma_{(\nu-1)}\|)<\eta$}
    \end{algorithmic}
\end{algorithm}

\begin{remark*}
    {\rm
    In \ref{alg:enum} of Algorithm \ref{alg:SGD}, we call the momentum SGD \cite{rumelhart1986learning}, in which the following expression is substituted for the update expression.
    \begin{equation*}
        \alpha_l^{(\mu+1)}:=\alpha_l^{(\mu)}-\tau G[\alpha]_l^{(\mu)}+\tau_1(\alpha_l^{(\mu)}-\alpha_l^{(\mu-1)})
    \end{equation*}
    \begin{equation*}
      \beta_l^{(\mu+1)}:=\beta_l^{(\mu)}-\tau G[\beta]_l^{(\mu)}+\tau_1(\beta_l^{(\mu)}-\beta_l^{(\mu-1)})
    \end{equation*}
    \begin{equation*}
      \gamma_l^{(\mu+1)}:=\gamma_l^{(\mu)}-\tau G[\gamma]_l^{(\mu)}+\tau_1(\gamma_l^{(\mu)}-\gamma_l^{(\mu-1)})
    \end{equation*}
    where $\tau$ is the learning rate and $\tau_1$ is the momentum rate.
    }
\end{remark*}

\section{Numerical results}\label{sec4}
\setcounter{equation}{0}

\subsection{Sinusoidal Curve}\label{sec4.1}
We performed a numerical example of the regression problem of a 1-dimensional signal $F(\xi)=\sin4\pi\xi$ defined on $\xi\in[0,1]$. Let the number of training data be $K_1=1000$, and let the training data be
\begin{equation*}
    \left\{\left(\frac{k-1}{K_1},F\left(\frac{k-1}{K_1}\right)\right)\right\}_{k=1}^{K_1}\subset[0,1]\times\mathbb{R},\quad D_1:=\left\{\frac{k-1}{K_1}\right\}_{k=1}^{K_1}.
\end{equation*}
We run Algorithm \ref{alg:SGD} until $\nu=10000$. We set the learning rate to $\tau=0.01$, $A$ to the matrix with ones on the diagonal and zeros elsewhere, and
\begin{equation*}
    \alpha_{(0)}\equiv0,\quad\beta_{(0)}\equiv0,\quad\gamma_{(0)}\equiv0.
\end{equation*}

Let the number of validation data be $K_2=3333$. The signal sampled with $\Delta\xi=1/K_2$ was used as the validation data. Let $D_2$ be the set of input data used for the validation data. Fig. \ref{fig:training-data-sin}. shows the training data which is $F(\xi)=\sin4\pi\xi$ sampled from $[0,1]$ with $\Delta\xi=1/K_1$. Fig. \ref{fig:predict-sin}. shows the result predicted using validation data when $\nu=10000$. The validation data is shown as a blue line, and the result predicted using the validation data is shown as an orange line. Fig. \ref{fig:initial-parameter-sin}. shows the initial values of parameters $\alpha,\beta$ and $\gamma$. Fig. \ref{fig:parameter-sin}. shows the learning results of each design parameter at $\nu=10000$. Fig. \ref{fig:loss-sin}. shows the change in the loss function during learning for each of the training data and validation data.

Fig. \ref{fig:loss-sin}. shows that the loss function can be decreased using Algorithm \ref{alg:SGD}. Fig. \ref{fig:predict-sin}. suggests that the prediction is good. In addition, the learning results of the parameters $\alpha,\beta$ and $\gamma$ are continuous functions.

\begin{figure}[H]
    \begin{minipage}{0.49\hsize}
        \centering
        \includegraphics[bb=0 0 432 288,width=\linewidth]{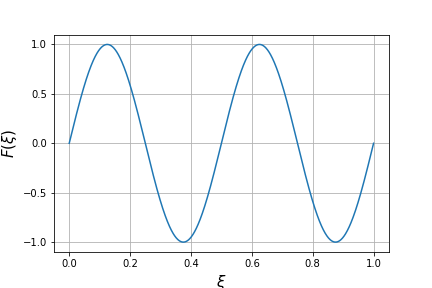}
        \caption{The training data which is $F(\xi)=\sin4\pi\xi$ sampled from $[0,1]$ with $\Delta\xi=1/K_1$.}
        \label{fig:training-data-sin}
    \end{minipage}
    \begin{minipage}{0.49\hsize}
        \centering
        \includegraphics[bb=0 0 432 288,width=\linewidth]{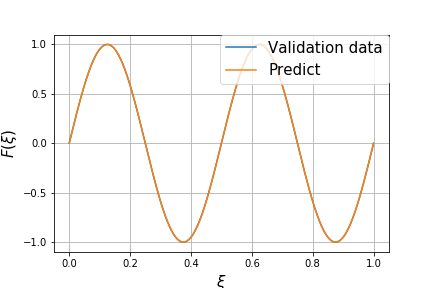}
        \caption{The result predicted using validation data when $\nu=10000$.}
        \label{fig:predict-sin}
    \end{minipage}
\end{figure}
\vspace{-0.8cm}
\begin{figure}[H]
    \begin{minipage}{0.49\hsize}
        \centering
        \includegraphics[bb=0 0 432 288,width=\linewidth]{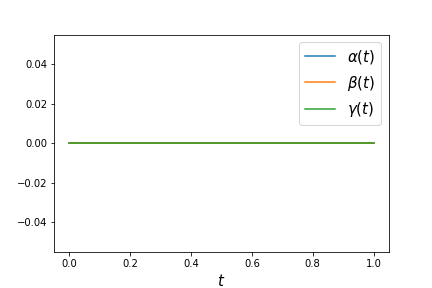}
        \caption{The initial values of design parameters $\alpha,\beta$ and $\gamma$ at each $t=hl$.}
        \label{fig:initial-parameter-sin}
    \end{minipage}
    \begin{minipage}{0.49\hsize}
        \centering
        \includegraphics[bb=0 0 432 288,width=\linewidth]{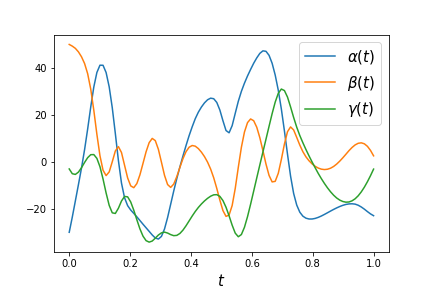}
        \caption{The learning results of design parameters $\alpha,\beta$ and $\gamma$ at each $t=hl$ when $\nu=10000$.}
        \label{fig:parameter-sin}
    \end{minipage}
\end{figure}
\vspace{-0.8cm}
\begin{figure}[H]
    \centering
    \includegraphics[bb=0 0 432 288,width=0.5\linewidth]{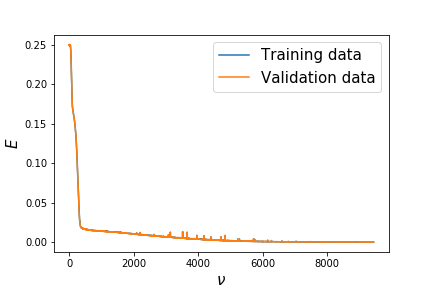}
    \caption{The change in the loss function during learning.}
    \label{fig:loss-sin}
\end{figure}

\subsection{Binary classification}\label{sec4.2}
We performed numerical experiments on a binary classification problem for 2-dimensional input. We set $n=2$ and $m=1$. Let the number of the training data be $K_1=10000$, and let $D_1=\{\xi^{(k)}|k=1,2,\ldots,K_1\}\subset [0,1]^2$ be the set of randomly generated points. Let
\begin{equation*}
    \left\{\left(\xi^{(k)},F(\xi^{(k)})\right)\right\}_{k=1}^{K_1}\subset[0,1]^2\times\mathbb{R},
\end{equation*}
\begin{equation}\label{eq:training-data-circle}
    F(\xi)=\left\{\begin{array}{ll}
        0, & \mathrm{if}~|\xi-(0.5,0.5)|<0.3, \\
        1, & \mathrm{if}~|\xi-(0.5,0.5)|\geq0.3,
    \end{array}\right.
\end{equation}
be the training data. We run Algorithm \ref{alg:SGD} until $\nu=10000$. We set the learning rate to $\tau=0.01$, $A$ to the matrix with ones on the diagonal and zeros elsewhere, and
\begin{equation*}
    \alpha_{(0)}\equiv0,\quad\beta_{(0)}\equiv0,\quad\gamma_{(0)}\equiv0.
\end{equation*}

Let the number of validation data be $K_2=2500$. The set of points $\xi$ randomly generated on $[0,1]^2$ and $F(\xi)$ is used as the validation data. Fig. \ref{fig:training-data-circle}. shows the training data in which randomly generated $\xi\in D_1$ are classified in \eqref{eq:training-data-circle}. Fig. \ref{fig:predict-circle}. shows the prediction result using validation data at $\nu=10000$. The results that were successfully predicted are shown in dark red and dark blue, and the results that were incorrectly predicted are shown in light red and light blue. Fig. \ref{fig:predict-circle-knn}. shows the result of predicting the validation data using $k$-nearest neighbor algorithm at $k=3$. Fig. \ref{fig:predict-circle-mlp}. shows the result of predicting the validation data using a multi-layer perceptron with $5000$ nodes. Fig. \ref{fig:initial-parameter-circle}. shows the initial value of parameters $\alpha,\beta$ and $\gamma$. Fig. \ref{fig:parameter-alpha-circle}., \ref{fig:parameter-beta-circle}. and \ref{fig:parameter-gamma-circle}. show the learning results of each parameters at $\nu=10000$. Fig. \ref{fig:loss-circle}. shows the change of the loss function during learning for each of the training data and validation data. Fig. \ref{fig:accuracy-circle}. shows the change of accuracy during learning. The accuracy is defined as
\begin{equation*}
    \mathrm{Accuracy}=\frac{\#\{\xi|F(\xi)=\bar{y}(\xi)\}}{K_i},\quad\mathrm{if}~\{\xi|F(\xi)=\bar{y}(\xi)\}\subset D_i,\quad(i=1,2),
\end{equation*}
\begin{equation*}
    \bar{y}(\xi):=\left\{\begin{array}{ll}
        0, & \mathrm{if}~y(T;\xi)<0.5, \\
        1, & \mathrm{if}~y(T;\xi)\geq0.5.
    \end{array}\right.
\end{equation*}
Table \ref{table:ex2-loss-accuracy} shows the value of the loss function and the accuracy of the prediction of each method.

From Fig. \ref{fig:loss-circle}. and \ref{fig:accuracy-circle}., we observe that the loss function can be decreased and accuracy can be increased using Algorithm \ref{alg:SGD}. Fig. \ref{fig:predict-circle}. shows that some points in the neighborhood of $|\xi-(0.5,0.5)|=0.3$ are wrongly predicted; however, most points are well predicted. The results are similar when compared with Fig. \ref{fig:predict-circle-knn}. and \ref{fig:predict-circle-mlp}. In addition, the learning results of the parameters $\alpha,\beta$, and $\gamma$ are continuous functions. From Table \ref{table:ex2-loss-accuracy}, the $k$-nearest neighbor algorithm minimizes the value of the loss function among the three methods. We consider that this is because the output of ODENet is $y(T;\xi)\in[0,1]$, while the output of the $k$-nearest neighbor algorithm is $\{0,1\}$. Compared to K-NN and MLP, the $(\alpha,\beta,\gamma)$-type ODENet gives a slightly worse result (Table \ref{table:ex2-loss-accuracy}). However, the binary classification is not a suitable problem to test the potential of ODENet, because the output values of ODENet are continuous. It should also be noted that the proposed model is not very tuned to increase accuracy. Considering these facts, this result also shows the potential of ODENet. 

\begin{figure}[H]
    \begin{minipage}{0.49\hsize}
        \centering
        \includegraphics[bb=0 0 276 307, width=0.7\linewidth]{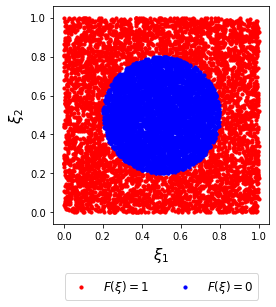}
        \caption{The training data defined by \eqref{eq:training-data-circle}.}
        \label{fig:training-data-circle}
    \end{minipage}
    \begin{minipage}{0.49\hsize}
        \centering
        \includegraphics[bb=0 0 361 330, width=\linewidth]{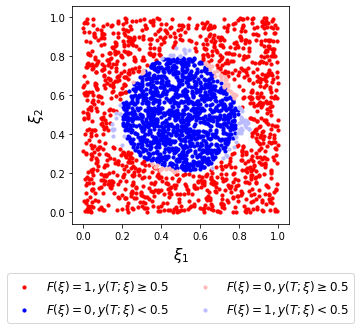}
        \caption{The result predicted using validation data when $\nu=10000$.}
        \label{fig:predict-circle}
    \end{minipage}
\end{figure}
\vspace{-0.8cm}
\begin{figure}[H]
    \begin{minipage}{0.49\hsize}
        \centering
        \includegraphics[bb=0 0 361 330, width=\linewidth]{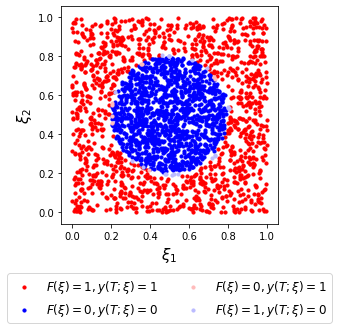}
        \caption{The result of predicting the validation data using $k$-nearest neighbor algorithm at $k=3$.}
        \label{fig:predict-circle-knn}
    \end{minipage}
    \begin{minipage}{0.49\hsize}
        \centering
        \includegraphics[bb=0 0 361 330, width=\linewidth]{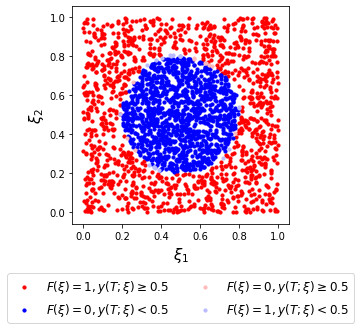}
        \caption{The result of predicting the validation data using a multi-layer perceptron with 5000 nodes.}
        \label{fig:predict-circle-mlp}
    \end{minipage}
\end{figure}
\vspace{-0.8cm}
\begin{figure}[H]
    \begin{minipage}{0.49\hsize}
        \centering
        \includegraphics[bb=0 0 432 288, width=\linewidth]{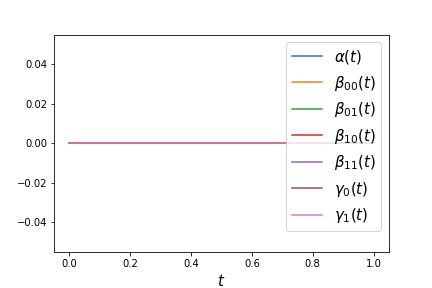}
        \caption{The initial values of design parameters $\alpha,\beta$ and $\gamma$ at each $t=hl$.}
        \label{fig:initial-parameter-circle}
    \end{minipage}
    \begin{minipage}{0.49\hsize}
        \centering
        \includegraphics[bb=0 0 432 288, width=\linewidth]{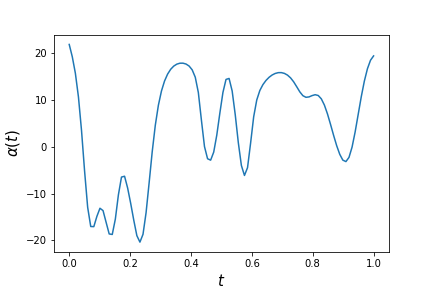}
        \caption{The learning result of design parameters $\alpha$ at each $t=hl$ when $\nu=10000$.}
        \label{fig:parameter-alpha-circle}
    \end{minipage}
\end{figure}
\begin{figure}[H]
    \begin{minipage}{0.49\hsize}
        \centering
        \includegraphics[bb=0 0 432 288, width=\linewidth]{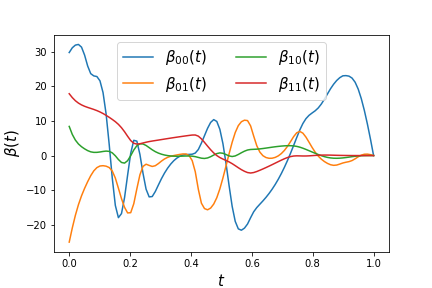}
        \caption{The learning result of design parameters $\beta$ at each $t=hl$ when $\nu=10000$.}
        \label{fig:parameter-beta-circle}
    \end{minipage}
    \begin{minipage}{0.49\hsize}
        \centering
        \includegraphics[bb=0 0 432 288, width=\linewidth]{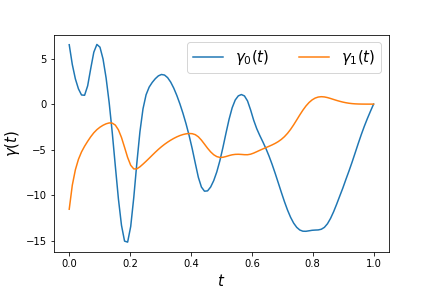}
        \caption{The learning result of design parameters $\gamma$ at each $t=hl$ when $\nu=10000$.}
        \label{fig:parameter-gamma-circle}
    \end{minipage}
\end{figure}
\begin{figure}[H]
    \begin{minipage}{0.49\hsize}
        \centering
        \includegraphics[bb=0 0 432 288, width=\linewidth]{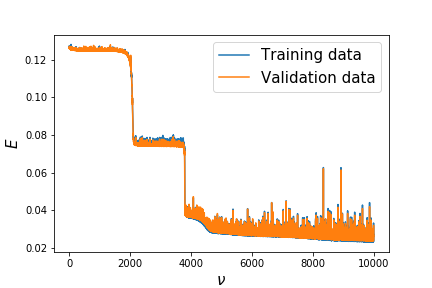}
        \caption{The change of the loss function during learning.}
        \label{fig:loss-circle}
    \end{minipage}
    \begin{minipage}{0.49\hsize}
        \centering
        \includegraphics[bb=0 0 432 288, width=\linewidth]{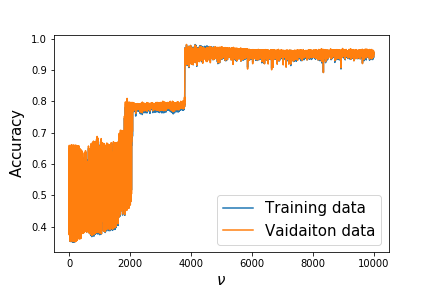}
        \caption{The change of accuracy during learning.}
        \label{fig:accuracy-circle}
    \end{minipage}
\end{figure}

\begin{table}[H]
    \begin{center}
        \caption{The prediction result of each method.}
        \label{table:ex2-loss-accuracy}
        \begin{tabular}{lll} \hline
            Method & Loss & Accuracy \\ \hline
            This paper ($(\alpha,\beta,\gamma)$-type ODENet) & 0.02629 & 0.9592 \\
            $K$-nearest neighbor algorithm (K-NN) & 0.006000 & 0.9879 \\
            Multilayer perceptron (MLP) & 0.006273 & 0.9883 \\ \hline
        \end{tabular}
    \end{center}
\end{table}

\subsection{Multinomial classification in MNIST}\label{sec4.3}
We performed a numerical experiment on a classification problem using MNIST, a dataset of handwritten digits. The input is a $28\times 28$ image and the output is a one-hot vector of labels attached to the MNIST dataset. We set $n=784$ and $m=10$. Let the number of training data be $K_1=43200$ and let the batchsize be $|I_{\mu}|=128$. We run Algorithm \ref{alg:SGD} until $\nu=1000$. However, the momentum SGD was used to update the design parameters. We set the learning rate as $\tau=0.01$, the momentum rate as $0.9$, $A$ to the matrix with ones on the diagonal and zeros elsewhere, and
\begin{equation*}
    \alpha_{(0)}\equiv10^{-8},\quad\beta_{(0)}\equiv10^{-8},\quad\gamma_{(0)}\equiv10^{-8}.
\end{equation*}

Let the number of validation data be $K_2=10800$. Fig. \ref{fig:loss-mnist}. shows the change of the loss function during learning for each of the training data and validation data. Fig. \ref{fig:accuracy-mnist}. shows the change of accuracy during learning. Using the test data, the values of the loss function and accuracy are
\begin{equation*}
    E=0.06432,\quad\mathrm{Accuracy}=0.9521,
\end{equation*}
at $\nu=1000$, respectively.

Fig. \ref{fig:loss-mnist}. and \ref{fig:accuracy-mnist}. suggest that the loss function can be decreased and accuracy can be increased using Algorithm \ref{alg:SGD} (using the Momentum SGD).

\begin{figure}[H]
    \begin{minipage}{0.49\hsize}
        \centering
        \includegraphics[bb=0 0 432 288, width=\linewidth]{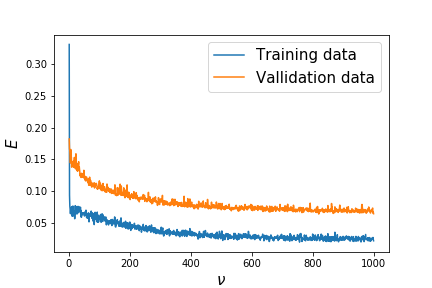}
        \caption{The change of the loss function during learning.}
        \label{fig:loss-mnist}
    \end{minipage}
    \begin{minipage}{0.49\hsize}
        \centering
        \includegraphics[bb=0 0 432 288, width=\linewidth]{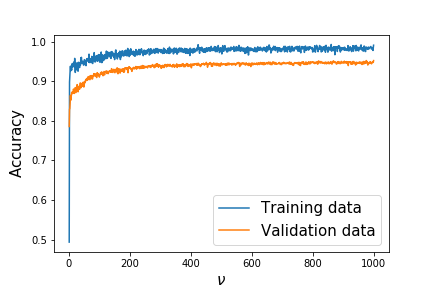}
        \caption{The change of accuracy during learning.}
        \label{fig:accuracy-mnist}
    \end{minipage}
\end{figure}

\section{Conclusion}\label{sec5}
In this paper, we proposed the $(\alpha,\beta,\gamma)$-type ODENet and the $(\alpha,\beta,\gamma)$-type ResNet and showed that they uniformly approximate an arbitrary continuous function on a compact set. This result shows that the $(\alpha,\beta,\gamma)$-type ODENet and the $(\alpha,\beta,\gamma)$-type ResNet can represent a variety of data. In addition, we showed the existence and continuity of the gradient of the loss function in a general ODENet. We performed numerical experiments on some data and confirmed that the gradient method reduces the loss function and represents the data.

Our future work is to show that the design parameters converge to a global minimizer of the loss function using a continuous gradient. We also plan to show that ODENet with other forms, such as convolution, can represent arbitrary data.

\section{Acknowledgement}\label{sec6}
This work is partially supported by JSPSKAKENHI JP20KK0058, and JST CREST Grant Number JPMJCR2014,
Japan.

\bibliographystyle{plain}
\bibliography{references}

\begin{thebibliography}{10}

\bibitem{attali1997approximations}
J.-G. Attali and G.~Pag\`es.
\newblock Approximations of functions by a multilayer perceptron: a new
  approach.
\newblock {\em Neural Networks}, 10(6):1069--1081, 1997.

\bibitem{baker2003matrix}
A.~Baker.
\newblock {\em Matrix groups: An Introduction to Lie Group Theory}.
\newblock Springer Science \& Business Media, 2003.

\bibitem{bengio1994learning}
Y.~Bengio, P.~Simard, and P.~Frasconi.
\newblock Learning long-term dependencies with gradient descent is difficult.
\newblock {\em IEEE Transactions on Neural Networks}, 5(2):157--166, 1994.

\bibitem{bottou1998online}
L.~Bottou.
\newblock Online algorithms and stochastic approximations.
\newblock In {\em Online Learning and Neural Networks}. Cambridge University
  Press, 1998.

\bibitem{carroll1989construction}
S.M. Carroll and B.W. Dickinson.
\newblock Construction of neural nets using the radon transform.
\newblock In {\em International 1989 Joint Conference on Neural Networks},
  volume~1, pages 607--611. IEEE, 1989.

\bibitem{chen2018neural}
R.T.Q. Chen, Y.~Rubanova, J.~Bettencourt, and D.K. Duvenaud.
\newblock Neural ordinary differential equations.
\newblock In {\em Advances in Neural Information Processing Systems},
  volume~31, pages 6571--6583, 2018.

\bibitem{cybenko1989approximation}
G.~Cybenko.
\newblock Approximation by superpositions of a sigmoidal function.
\newblock {\em Mathematics of Control, Signals and Systems}, 2(4):303--314,
  1989.

\bibitem{fukushima1982neocognitron}
K.~Fukushima and S.~Miyake.
\newblock Neocognitron: A self-organizing neural network model for a mechanism
  of visual pattern recognition.
\newblock In {\em Competition and Cooperation in Neural Nets}, pages 267--285.
  Springer, 1982.

\bibitem{funahashi1989approximate}
K.-I. Funahashi.
\newblock On the approximate realization of continuous mappings by neural
  networks.
\newblock {\em Neural Networks}, 2(3):183--192, 1989.

\bibitem{glorot2010understanding}
X.~Glorot and Y.~Bengio.
\newblock Understanding the difficulty of training deep feedforward neural
  networks.
\newblock In {\em Proceedings of the 13th International Conference on
  Artificial Intelligence and Statistics}, volume~9, pages 249--256. PMLR,
  2010.

\bibitem{hanin2017approximating}
B.~Hanin and M.~Sellke.
\newblock Approximating continuous functions by {ReLU} nets of minimal width.
\newblock {\em arXiv preprint arXiv:1710.11278}, 2017.

\bibitem{he2015convolutional}
K.~He and J.~Sun.
\newblock Convolutional neural networks at constrained time cost.
\newblock In {\em 2015 IEEE Conference on Computer Vision and Pattern
  Recognition (CVPR)}, pages 5353--5360. IEEE, 2015.

\bibitem{he2016deep}
K.~He, X.~Zhang, S.~Ren, and J.~Sun.
\newblock Deep residual learning for image recognition.
\newblock In {\em 2016 IEEE Conference on Computer Vision and Pattern
  Recognition (CVPR)}, pages 770--778. IEEE, 2016.

\bibitem{hornik1989multilayer}
K.~Hornik, M.~Stinchcombe, and H.~White.
\newblock Multilayer feedforward networks are universal approximators.
\newblock {\em Neural Networks}, 2(5):359--366, 1989.

\bibitem{kidger2020universal}
P.~Kidger and T.~Lyons.
\newblock {Universal Approximation with Deep Narrow Networks}.
\newblock In {\em Proceedings of 33rd Conference on Learning Theory}, pages
  2306--2327. PMLR, 2020.

\bibitem{leshno1993multilayer}
M.~Leshno, V.Y. Lin, A.~Pinkus, and S.~Schocken.
\newblock Multilayer feedforward networks with a nonpolynomial activation
  function can approximate any function.
\newblock {\em Neural Networks}, 6(6):861--867, 1993.

\bibitem{lin2018resnet}
H.~Lin and S.~Jegelka.
\newblock {ResNet} with one-neuron hidden layers is a universal approximator.
\newblock In {\em Advances in Neural Information Processing Systems},
  volume~31, pages 6169--6178. Curran Associates, Inc., 2018.

\bibitem{mcculloch1943logical}
W.S. McCulloch and W.~Pitts.
\newblock A logical calculus of the ideas immanent in nervous activity.
\newblock {\em The bulletin of mathematical biophysics}, 5:115--133, 1943.

\bibitem{rumelhart1986learning}
D.E. Rumelhart, G.E. Hinton, and R.J. Williams.
\newblock Learning representations by back-propagating errors.
\newblock {\em Nature}, 323(6088):533--536, 1986.

\bibitem{schmidhuber2015deep}
J.~Schmidhuber.
\newblock Deep learning in neural networks: An overview.
\newblock {\em Neural Networks}, 61:85--117, 2015.

\bibitem{simonyan2014very}
K.~Simonyan and A.~Zisserman.
\newblock Very deep convolutional networks for large-scale image recognition.
\newblock In {\em International Conference on Learning Representations}, 2015.

\bibitem{sonoda2017neural}
S.~Sonoda and N.~Murata.
\newblock Neural network with unbounded activation functions is universal
  approximator.
\newblock {\em Applied and Computational Harmonic Analysis}, 43(2):233--268,
  2017.

\bibitem{szegedy2015going}
C.~Szegedy, W.~Liu, Y.~Jia, P.~Sermanet, S.~Reed, D.~Anguelov, D.~Erhan,
  V~Vanhoucke, and A~Rabinovich.
\newblock Going deeper with convolutions.
\newblock In {\em 2015 IEEE Conference on Computer Vision and Pattern
  Recognition (CVPR)}, pages 1--9, 2015.

\end{thebibliography}

\appendix
\renewcommand{\thesection}{{{\Alph{section}}}}
\renewcommand{\theequation}{\Alph{section}.\arabic{equation}}

\section{Differentiability with respect to parameters of ODE}\label{appendix1}
\setcounter{equation}{0}
We discuss the differentiability with respect to the design parameters of ordinary differential equations.
\begin{theorem}\label{thm:differentiability-ode}
    {\rm
    Let $N$ and $r$ be natural numbers, and $T$ be a positive real number. We define $X:=C^1([0,T];\mathbb{R}^N)$ and $\Omega:=C([0,T];\mathbb{R}^r)$. We consider the initial value problem for the ordinary differential equation:
    \begin{equation}\label{eq:ode-ivp}
        \left\{\begin{aligned}
            x'(t)&=f(t,x(t),\omega(t)), & t\in(0,T], \\
            x(0)&=\xi, &
        \end{aligned}\right.
    \end{equation}
    where $x$ is a function from $[0,T]$ to $\mathbb{R}^N$, and $\xi\in D$ is the initial value; $\omega\in\Omega$ is the design parameter; $f$ is a continuously differentiable function from $[0,T]\times\mathbb{R}^N\times\mathbb{R}^r$ to $\mathbb{R}^N$; There exists $L>0$ such that
    \begin{equation*}
        |f(t,x_1,\omega(t))-f(t,x_2,\omega(t))|\leq L|x_1-x_2|
    \end{equation*}
    for any $t\in[0,T],x_1,x_2\in\mathbb{R}^N$, and $\omega\in\Omega$. Then, the solution to \eqref{eq:ode-ivp} satisfies $[\Omega\ni\omega\mapsto x[\omega]]\in C^1(\Omega;X)$. Furthermore, if we define
    \begin{equation*}
        y(t):=(\partial_{\omega}x[\omega]\eta)(t)
    \end{equation*}
    for any $\eta\in\Omega$, the following relations
    \begin{equation*}
        \left\{\begin{array}{ll}
            y'(t)-\nabla_x^{\top}f(t,x[\omega](t),\omega(t))y(t)=\nabla_{\omega}^{\top}f(t,x[\omega](t),\omega(t))\eta(t), & t\in(0,T], \\
            y(0)=0, &
        \end{array}\right.
    \end{equation*}
    are satisfied.
    }
\end{theorem}
\begin{proof}
    Let $X_0$ be the set of continuous functions from $[0,T]$ to $\mathbb{R}^N$. Because $f(t,\cdot,\omega(t))$ is Lipschitz continuous for any $t\in[0,T]$ and $\omega\in\Omega$, there exists a unique solution $x[\omega]\in X$ in \eqref{eq:ode-ivp}. We define the map $J:X\times\Omega\to X$ as
    \begin{equation*}
        J(x,\omega)(t):=x(t)-\xi-\int_0^tf(s,x(s),\omega(s))ds.
    \end{equation*}
    The map $J$ satisfies
    \begin{equation*}
        J(x,\omega)'(t)=x'(t)-f(t,x(t),\omega(t)).
    \end{equation*}
    Since $f\in C^1([0,T]\times\mathbb{R}^N\times\mathbb{R}^r;\mathbb{R}^N)$, $J\in C(X\times\Omega;X)$.

    Take an arbitrary $\omega\in\Omega$. For any $x\in X$, let
    \begin{equation*}
        f\circ x(t):=f(t,x(t),\omega(t)),\quad\nabla_x^{\top}f\circ x(t):=\nabla_x^{\top}f(t,x(t),\omega(t)).
    \end{equation*}
    We define the map $A(x):X\to X$ as
    \begin{equation*}
        (A(x)y)(t):=y(t)-\int_0^t(\nabla_x^{\top}f\circ x(s))y(s)ds.
    \end{equation*}
    The map $A(x)$ satisfies
    \begin{equation*}
        (A(x)y)'(t)=y'(t)-(\nabla_x^{\top}f\circ x(t))y(t).
    \end{equation*}
    $x$ and $\omega$ are bounded because they are continuous functions on a compact interval. Because $f\in C^1([0,T]\times\mathbb{R}^N\times\mathbb{R}^r;\mathbb{R}^N)$, there exists $C>0$ such that $|\nabla_x^{\top}f\circ x(t)|\leq C$ for any $t\in[0,T]$. From
    \begin{equation*}
        |(A(x)y)(t)|\leq\|y\|_{X_0}+CT\|y\|_{X_0},\quad|(A(x)y)'(t)|\leq\|y'\|_{X_0}+C\|y\|_{X_0},
    \end{equation*}
    \begin{equation*}
        \|A(x)y\|_X\leq\|y\|_X+C(T+1)\|y\|_X=(1+C(T+1))\|y\|_X
    \end{equation*}
    is satisfied. Hence, $A(x)\in B(X,X)$. Let us fix $x_0\in X$. We take $x\in X$ such that $x\to x_0$.
    \begin{align*}
        |(A(x)y)(t)-(A(x_0)y)(t)| &\leq \int_0^t|\nabla_x^{\top}f\circ x(s)-\nabla_x^{\top}f\circ x_0(s)||y(s)|ds \\
        &\leq T\|y\|_X\|\nabla_x^{\top}f\circ x-\nabla_x^{\top}f\circ x_0\|_{C([0,T];\mathbb{R}^{N\times N})}
    \end{align*}
    \begin{align*}
        |(A(x)y)'(t)-(A(x_0)y)'(t)| &\leq |\nabla_x^{\top}f\circ x(t)-\nabla_x^{\top}f\circ x_0(t)||y(t)| \\
        &\leq \|y\|_X\|\nabla_x^{\top}f\circ x-\nabla_x^{\top}f\circ x_0\|_{C([0,T];\mathbb{R}^{N\times N})}
    \end{align*}
    \begin{equation*}
        \|A(x)y-A(x_0)y\|_X\leq(T+1)\|y\|_X\|\nabla_x^{\top}f\circ x-\nabla_x^{\top}f\circ x_0\|_{C([0,T];\mathbb{R}^{N\times N})}
    \end{equation*}
    \begin{equation*}
        \|A(x)-A(x_0)\|_{B(X,X)}\leq(T+1)\|\nabla_x^{\top}f\circ x-\nabla_x^{\top}f\circ x_0\|_{C([0,T];\mathbb{R}^{N\times N})}
    \end{equation*}
    Hence, $A\in C(X;B(X,X))$.
    \begin{align*}
        &J(x+y,\omega)(t)-J(x,\omega)(t)-(A(x)y)(t) \\
        &=-\int_0^t(f\circ(x+y)(s)-f\circ x(s)-(\nabla_x^{\top}f\circ x(s))y(s))ds
    \end{align*}
    \begin{equation*}
        \|J(x+y,\omega)-J(x,\omega)-A(x)y\|_X\leq(T+1)\|f\circ(x+y)-f\circ x-(\nabla_x^{\top}f\circ x)y\|_{X_0}
    \end{equation*}
    Form the Taylor expansion of $f$, we obtain
    \begin{equation*}
        f(t,x(t)+y(t),\omega(t))=f(t,x(t),\omega(t))+\int_0^1\nabla_x^{\top}f(t,x(t)+\zeta y(t),\omega(t))y(t)d\zeta
    \end{equation*}
    for any $t\in[0,T],x,y\in X$ and $\omega\in\Omega$. We obtain
    \begin{equation*}
        |f\circ(x+y)(t)-f\circ x(t)-(\nabla_x^{\top}f\circ x(t))y(t)|\leq\int_0^1|\nabla_x^{\top}f\circ(x+\zeta y)(t)-\nabla_x^{\top}f\circ x(t)||y(t)|d\zeta.
    \end{equation*}
    For any $\epsilon>0$, there exists $\delta>0$ such that
    \begin{equation*}
        \|y\|_{X_0}<\delta,\zeta\in[0,1]~\Rightarrow~|\nabla_x^{\top}f\circ(x+\zeta y)(t)-\nabla_x^{\top}f\circ x(t)|<\epsilon.
    \end{equation*}
    We obtain
    \begin{equation*}
        |f\circ(x+y)(t)-f\circ x(t)-(\nabla_x^{\top}f\circ x(t))y(t)|\leq\epsilon\|y\|_{X_0},
    \end{equation*}
    \begin{equation*}
        \|J(x+y,\omega)-J(x,\omega)-A(x)y\|_X\leq\epsilon(T+1)\|y\|_X.
    \end{equation*}
    Hence,
    \begin{equation*}
        \partial_xJ(x,\omega)y=A(x)y.
    \end{equation*}
    From $\partial_xJ(\cdot,\omega)\in C(X;B(X,X))$, $J(\cdot,\omega)\in C^1(X;X)$.

    By fixing $\omega_0\in\Omega$, there exists a solution $x_0\in X$ of \eqref{eq:ode-ivp} such that
    \begin{equation*}
        x_0(t)=\xi+\int_0^tf(s,x_0(s),\omega_0(s))ds.
    \end{equation*}
    That is,
    \begin{equation*}
        J(x_0,\omega_0)(t)=x_0(t)-\xi-\int_0^tf(s,x_0(s),\omega_0(s))ds=x_0(t)-x_0(t)=0
    \end{equation*}
    is satisfied. If $y\in X$ satisfies $(\partial_xJ(x_0,\omega_0)y)(t)=g(t)$ for any $g\in X$, then
    \begin{equation*}
        \left\{\begin{array}{ll}
            y'(t)-\nabla_x^{\top}f(t,x_0(t),\omega_0(t))y(t)=g'(t), & t\in(0,T], \\
            y(0)=g(0). &
        \end{array}\right.
    \end{equation*}
    Because the solution to this ordinary differential equation exists uniquely, there exists an inverse map $(\partial_xJ(x_0,\omega_0))^{-1}$ such that $(\partial_xJ(x_0,\omega_0))^{-1}\in B(X,X)$.

    From the implicit function theorem, for any $\omega\in\Omega$, there exists $x[\omega]\in X$ such that $J(x[\omega],\omega)=0$. From $J\in C^1(X\times\Omega;X)$, we obtain $[\omega\mapsto x[\omega]]\in C^1(\Omega;X)$. We put
    \begin{equation*}
        y(t):=(\partial_{\omega}x[\omega]\eta)(t)
    \end{equation*}
    for any $\eta\in\Omega$. From $J(x[\omega],\omega)=0$,
    \begin{equation*}
        (\partial_xJ(x[\omega],\omega)y)(t)+(\partial_{\omega}J(x[\omega],\omega)\eta)(t)=0,
    \end{equation*}
    \begin{equation*}
        y(t)-\int_0^t\nabla_x^{\top}f(s,x[\omega](s),\omega(s))y(s)ds-\int_0^t\nabla_{\omega}^{\top}f(s,x[\omega](s),\omega(s))\eta(s)ds=0.
    \end{equation*}
    Therefore, we obtain
    \begin{equation*}
        \left\{\begin{array}{ll}
            y'(t)-\nabla_x^{\top}f(t,x[\omega](t),\omega(t))y(t)=\nabla_{\omega}^{\top}f(t,x[\omega](t),\omega(t))\eta(t), & t\in(0,T], \\
            y(0)=0. &
        \end{array}\right.
    \end{equation*}
\end{proof}

\section{General ODENet}\label{appendix2}
\setcounter{equation}{0}
In this section, we describe the general ODENet and the existence and continuity of the gradient of a loss function with respect to the design parameter. Let $N$ and $r$ be natural numbers and $T$ be a positive real number. Let the input data $D\subset\mathbb{R}^n$ be a compact set. We define $X:=C^1([0,T];\mathbb{R}^N)$ and $\Omega:=C([0,T];\mathbb{R}^r)$. We consider the ODENet with the following system of ordinary differential equations.
\begin{equation}\label{eq:odenet-general}
    \left\{\begin{aligned}
        x'(t)&=f(t,x(t),\omega(t)), & t\in(0,T], \\
        x(0)&=Q\xi, &
    \end{aligned}\right.
\end{equation}
where $x$ is a function from $[0,T]$ to $\mathbb{R}^N$; $\xi\in D$ is the input data; $Px(T)$ is the final output; $\omega\in\Omega$ is the design parameter; $P$ and $Q$ are $m\times N$ and $N\times n$ real matrices; $f$ is a continuously differentiable function from $[0,T]\times\mathbb{R}^N\times\mathbb{R}^r$ to $\mathbb{R}^N$, and $f(t,\cdot,\omega(t))$ is Lipschitz continuous for any $t\in[0,T]$ and $\omega\in\Omega$. For an input data $\xi\in D$, we denote the output data as $Px(T;\xi)$. We consider an approximation of $F\in C(D;\mathbb{R}^m)$ using ODENet with a system of ordinary differential equations \eqref{eq:odenet-general}. We define the loss function as
\begin{equation*}
    e[x]=\frac{1}{2}\left|Px(T;\xi)-F(\xi)\right|^2.
\end{equation*}
We define the gradient of the loss function with respect to the design parameter as follows:

\begin{definition}\label{def:gradient}
    {\rm
    Let $\Omega$ be a real Banach space. Assume that the inner product $\left<\cdot,\cdot\right>$ is defined on $\Omega$. The functional $\Phi:\Omega\to\mathbb{R}$ is a Fréchet differentiable at $\omega\in\Omega$. The Fréchet derivative is denoted by $\partial\Phi[\omega]\in\Omega^{*}$. If $G[\omega]\in\Omega$ exists such that
    \begin{equation*}
        \partial\Phi[\omega]\eta=\left<G[\omega],\eta\right>,
    \end{equation*}
    for any $\eta\in\Omega$, we call $G[\omega]$ the gradient of $\Phi$ at $\omega\in\Omega$ with respect to the inner product $\left<\cdot,\cdot\right>$.
    }
\end{definition}

\begin{remark*}
    {\rm
    If there exists a gradient $G[\omega]$ of the functional $\Phi$ at $\omega\in\Omega$ with respect to the inner product $\left<\cdot,\cdot\right>$, the algorithm to find the minimum value of $\Phi$ by
    \begin{equation*}
        \omega_{(\nu)}=\omega_{(\nu-1)}-\tau G[\omega_{(\nu-1)}]
    \end{equation*}
    is called the steepest descent method.
    }
\end{remark*}

\begin{theorem}\label{thm:general-adjoint-gradient}
    {\rm
    Given the design parameter $\omega\in\Omega$, let $x[\omega](t;\xi)$ be the solution to \eqref{eq:odenet-general} with the initial value $\xi\in D$. Let $\lambda:[0,T]\to\mathbb{R}^N$ be the adjoint and satisfy the following adjoint equation:
    \begin{equation*}
        \left\{\begin{aligned}
            \lambda'(t)&=-\nabla_xf^{\top}\left(t,x[\omega](t;\xi),\omega(t)\right)\lambda(t), & t\in[0,T), \\
            \lambda(T)&=P^{\top}\left(Px[\omega](T;\xi)-F(\xi)\right). &
        \end{aligned}\right.
    \end{equation*}
    We define the functional $\Phi:\Omega\to\mathbb{R}$ as $\Phi[\omega]=e[x[\omega]]$. Then, there exists a gradient $G[\omega]\in\Omega$ of $\Phi$ as $\omega\in\Omega$ with respect to the $L^2(0,T;\mathbb{R}^r)$ inner predict such that
    \begin{equation*}
        \partial\Phi[\omega]\eta=\int_0^TG[\omega](t)\cdot\eta(t)dt,\quad G[\omega](t)=\nabla_{\omega}f^{\top}\left(t,x[\omega](t;\xi),\omega(t)\right)\lambda(t),
    \end{equation*}
    for any $\eta\in\Omega$.
    }
\end{theorem}

\begin{proof}
    $e$ is a continuously differentiable function from $X$ to $\mathbb{R}$, and the solution of \eqref{eq:odenet-general} satisfies $[\omega\mapsto x[\omega]]$ from the Theorem \ref{thm:differentiability-ode}. Hence, $\Phi\in C^1(\Omega)$. For any $\eta\in\Omega$,
    \begin{align*}
        \partial\Phi[\omega]\eta &= (Px[\omega](T;\xi)-F(\xi))\cdot P(\partial_{\omega}x[\omega]\eta)(T), \\&
        = P^{\top}(Px[\omega](T;\xi)-F(\xi))\cdot(\partial_{\omega}x[\omega]\eta)(T).
    \end{align*}
    We put $y(t):=(\partial_{\omega}x[\omega]\eta)(t)$. From Theorem \ref{thm:differentiability-ode}, we obtain
    \begin{equation*}
        \left\{\begin{array}{ll}
            y'(t)-\nabla_x^{\top}f\left(t,x[\omega](t,\xi),\omega(t)\right)y(t)=\nabla_{\omega}^{\top}f\left(t,x[\omega](t;\xi),\omega(t)\right)\eta(t), & t\in(0,T], \\
            y(0)=0.
        \end{array}\right.
    \end{equation*}
    Since the assumption,
    \begin{equation*}
        \left\{\begin{aligned}
            \lambda'(t)&=-\nabla_xf^{\top}\left(t,x[\omega](t;\xi),\omega(t)\right)\lambda(t), & t\in[0,T), \\
            \lambda(T)&=P^{\top}\left(Px[\omega](T;\xi)-F(\xi)\right). &
        \end{aligned}\right.
    \end{equation*}
    is satisfied. We define
    \begin{equation*}
        g(t):=\nabla_{\omega}f^{\top}\left(t,x[\omega](t;\xi),\omega(t)\right)\lambda(t).
    \end{equation*}
    Then, $g\in\Omega$ is satisfied. We calculate the $L^2(0,T;\mathbb{R}^r)$ inner product of $g$ and $\eta$,
    \begin{align*}
        \left<g,\eta\right> &= \int_0^T(\nabla_{\omega}f^{\top}(t,x[\omega](t;\xi),\omega(t))\lambda(t))\cdot\eta(t)dt, \\
        &= \int_0^T\lambda(t)\cdot(\nabla_{\omega}^{\top}f(t,x[\omega](t;\xi),\omega(t))\eta(t))dt, \\
        &= \int_0^T\lambda(t)\cdot(y'(t)-\nabla_x^{\top}f(t,x[\omega](t;\xi),\omega(t))y(t))dt, \\
        &= \lambda(T)\cdot y(T)-\lambda(0)\cdot y(0)-\int_0^T(\lambda'(t)+\nabla_xf^{\top}(t,x[\omega](t;\xi),\omega(t))\lambda(t))\cdot y(t)dt, \\
        &= P^{\top}(Px[\omega](t;\xi)-F(\xi))\cdot y(T), \\
        &= \partial\Phi[\omega]\eta.
    \end{align*}
    Therefore, there exists a gradient $G[\omega]\in\Omega$ of $\Phi$ at $\omega\in\Omega$ with respect to the $L^2(0,T;\mathbb{R}^r)$ inner product such that
    \begin{equation*}
        G[\omega](t)=\nabla_{\omega}f^{\top}\left(t,x[\omega](t;\xi),\omega(t)\right)\lambda(t).
    \end{equation*}
\end{proof}

\section{General ResNet}\label{appendix3}
\setcounter{equation}{0}
In this section, we describe the general ResNet and error backpropagation. We consider a ResNet with the following system of difference equations
\begin{equation}\label{eq:resnet-general}
    \left\{\begin{aligned}
        x^{(l+1)}&=x^{(l)}+f^{(l)}(x^{(l)},\omega^{(l)}), & l=0,1,\ldots,L-1, \\
        x^{(0)}&=Q\xi, &
    \end{aligned}\right.
\end{equation}
where $x^{(l)}$ is an $N$-dimensional real vector for all $l=0,1,\ldots,L$; $\xi\in D$ is the input data; $Px^{(L)}$ is the final output; $\omega^{(l)}\in\mathbb{R}^{r_l}~(l=0,1,\ldots,L-1)$ are the design parameters; $P$ and $Q$ are $m\times N$ and $N\times n$ real matrices; $f^{(l)}$ is a continuously differentiable function from $\mathbb{R}^N\times\mathbb{R}^{r_l}$ to $\mathbb{R}^N$ for all $l=0,1,\ldots,L-1$. We consider an approximation of $F\in C(D;\mathbb{R}^m)$ using ResNet with a system of difference equations \eqref{eq:resnet-general}. Let $K\in\mathbb{N}$ be the number of training data and $\{(\xi^{(k)},F(\xi^{(k)}))\}_{k=1}^K\subset D\times\mathbb{R}^m$ be the training data. We divide the label of the training data into the following disjoint sets.
\begin{equation*}
    \{1,2,\ldots,K\}=I_1\cup I_2\cup\cdots\cup I_M~(\mathrm{disjoint}),\quad(1\leq M\leq K).
\end{equation*}
Let $Px^{(L,k)}$ denote the final output for a given input data $\xi^{(k)}\in D$. We set $\mbox{\boldmath $\omega$}=(\omega^{(0)},\omega^{(1)},\ldots,\omega^{(L-1)})$. We define the loss function for all $\mu=1,2,\ldots,M$ as follows:
\begin{equation}\label{eq:loss-function-resnet-general}
  e_{\mu}(\mbox{\boldmath $\omega$})=\frac{1}{2|I_{\mu}|}\sum_{k\in I_{\mu}}\left|Px^{(L,k)}-F(\xi^{(k)})\right|^2,
\end{equation}
\begin{equation*}
  E=\frac{1}{2K}\sum_{k=1}^K\left|Px^{(L,k)}-F(\xi^{(k)})\right|^2.
\end{equation*}
We consider the learning for each label set using the gradient method. We find the gradient of the loss function \eqref{eq:loss-function-resnet-general} with respect to the design parameter $\omega^{(l)}\in\mathbb{R}^{r_l}$ for all $l=0,1,\ldots,L-1$ using error backpropagation. Using the chain rule, we obtain
\begin{equation*}
  \nabla_{\omega^{(l)}}e_{\mu}(\mbox{\boldmath $\omega$})=\sum_{k\in I_{\mu}}\nabla_{\omega^{(l)}}{x^{(l+1,k)}}^{\top}\nabla_{x^{(l+1,k)}}e_{\mu}(\mbox{\boldmath $\omega$})
\end{equation*}
for all $l=0,1,\ldots,L-1$. From \eqref{eq:resnet-general},
\begin{equation*}
  \nabla_{\omega^{(l)}}{x^{(l+1,k)}}^{\top}=\nabla_{\omega^{(l)}}{f^{(l)}}^{\top}(x^{(l,k)},\omega^{(l)}).
\end{equation*}
We define $\lambda^{(l,k)}:=\nabla_{x^{(l,k)}}e_{\mu}(\mbox{\boldmath $\omega$})$ for all $l=0,1,\ldots,L$ and $k\in I_{\mu}$. We obtain
\begin{equation*}
    \lambda^{(l,k)}=\nabla_{x^{(l,k)}}{x^{(l+1,k)}}^{\top}\nabla_{x^{(l+1,k)}}e_{\mu}(\mbox{\boldmath $\omega$})=\lambda^{(l+1,k)}+\nabla_{x^{(l,k)}}{f^{(l)}}^{\top}(x^{(l,k)},\omega^{(l)})\lambda^{(l+1,k)}.
\end{equation*}
Also,
\begin{equation*}
  \lambda^{(L,k)}=\nabla_{x^{(L,k)}}e_{\mu}(\mbox{\boldmath $\omega$})=\frac{1}{|I_{\mu}|}P^{\top}\left(Px^{(L,k)}-F(\xi^{(k)})\right).
\end{equation*}
Therefore, we can find the gradient $\nabla_{\omega^{(l)}}e_{\mu}(\mbox{\boldmath $\omega$})$ of the loss function \eqref{eq:loss-function-resnet-general} with respect to the design parameters $\omega^{(l)}\in\mathbb{R}^r$ by using the following equations
\begin{equation*}
  \left\{\begin{array}{lll}
      \displaystyle{\nabla_{\omega^{(l)}}e_{\mu}(\mbox{\boldmath $\omega$})=\sum_{k\in I_{\mu}}\nabla_{\omega^{(l)}}{f^{(l)}}^{\top}(x^{(l,k)},\omega^{(l)})\lambda^{(l+1,k)}}, & l=0,1,\ldots,L-1, & \\
      \displaystyle{\lambda^{(l,k)}=\lambda^{(l+1,k)}+\nabla_{x^{(l,k)}}{f^{(l)}}^{\top}(x^{(l,k)},\omega^{(l)})\lambda^{(l+1,k)}}, & l=0,1,\ldots,L-1, & k\in I_{\mu}, \\
      \displaystyle{\lambda^{(L,k)}=\frac{1}{|I_{\mu}|}P^{\top}\left(Px^{(L,k)}-F(\xi^{(k)})\right)}, & & k\in I_{\mu}.
  \end{array}\right.
\end{equation*}

\section{General ODENet and $(\alpha,\beta,\gamma)$-type ODENet}\label{appendix4}

In this section, we show that \eqref{eq:odenet-general} is a generalization of \eqref{eq:odenet-main}. Let natural numbers $n$ amd $m$ satisfy $n \ge m$. In \eqref{eq:odenet-general} with $N=m+n$ and $r=m+n(n+1)$, if we put 
\begin{align*}\begin{array}{c}
    {\displaystyle x(t) = \left(\begin{array}{c}
        \tilde{x}(t)\\ \tilde{y}(t)
    \end{array}\right)\in\mathbb{R}^n\times\mathbb{R}^m,\quad
    \omega(t) = \left(\alpha(t), \beta(t), \gamma(t)\right)\in\mathbb{R}^m\times\mathbb{R}^{n\times n}\times\mathbb{R}^n,}\\[16pt]
    {\displaystyle P=\left( O_{m\times n} ~ I_m \right),\quad 
    Q=\left(\begin{array}{c}
        I_n \\ O_{m\times n}
    \end{array}\right),\quad
    f(t,x(t),\omega(t)) = \left(\begin{array}{c}
        \beta(t)\tilde{x}(t)+\gamma(t)\\
        \alpha(t)\odot\mbox{\boldmath $\sigma$}(A\tilde{x}(t))
    \end{array}\right),}
\end{array}\end{align*}
where $A$ is an $m\times n$ real matrix, $I_n$ is the identity matrix of size $n$, and $O_{m\times n}$ is the $m\times n$ zero matrix, then $\tilde{x}(t)\in\mathbb{R}^n$ and $\tilde{y}(t)\in\mathbb{R}^m$ satisfy \eqref{eq:odenet-main} with the initial value 
\[
    \left(\begin{array}{c}
        \tilde{x}(0)\\ \tilde{y}(0)
    \end{array}\right)
    =Qx(0)
    =\left(\begin{array}{c} \xi\\ 0\end{array}\right)
\]
and $\tilde{y}(T)=Px(T)$ is the final output.

\end{document}